\documentclass{article}

\PassOptionsToPackage{table}{xcolor}

\usepackage{microtype}
\usepackage{graphicx}
\usepackage{subcaption}
\usepackage{booktabs}

\usepackage{hyperref}

\usepackage[accepted]{icml2026}

\usepackage{amsmath}
\usepackage{amssymb}
\usepackage{mathtools}
\usepackage{amsthm}

\usepackage[capitalize,noabbrev]{cleveref}

\theoremstyle{plain}
\newtheorem{theorem}{Theorem}[section]
\newtheorem{proposition}[theorem]{Proposition}

\theoremstyle{definition}

\newtheorem{assumption}[theorem]{Assumption}
\theoremstyle{remark}
\newtheorem{remark}[theorem]{Remark}

\usepackage[textsize=tiny]{todonotes}

\usepackage{xcolor}
\usepackage{algorithm}
\usepackage{algorithmic}
\usepackage{thmtools, thm-restate}

\usepackage{minitoc}
\usepackage{url}
\usepackage{etoc}
\usepackage{multirow}
\usepackage{array}
\usepackage{enumitem}
\usepackage{bbm}
\definecolor{lightgray}{gray}{0.9}
\definecolor{palegray}{gray}{0.98}
\definecolor{ma_blue}{HTML}{244983}
\definecolor{ma_green}{HTML}{2A8153}
\definecolor{ma_red}{HTML}{CF1C3A}
\definecolor{googleblue}{HTML}{4285F4}
\definecolor{googlered}{HTML}{EA4335}
\definecolor{googleyellow}{HTML}{FBBC05}
\definecolor{googlegreen}{HTML}{34A853}

\newenvironment{sproof}{%
  \proof}{\endproof}

\newenvironment{propappx}[2]{%
  \par\medskip\noindent\textbf{Proposition \ref{#1}} (#2)\textbf{.}\itshape}{\par\medskip}

\newcommand{\magigcode}{\textcolor{magenta}{\url{https://github.com/leekwoon/ma-gig/}}}

\icmltitlerunning{Manifold-Aligned Guided Integrated Gradients for Reliable Feature Attribution}

\begin{document}

\twocolumn[
	\icmltitle{Manifold-Aligned Guided Integrated Gradients for Reliable Feature Attribution}

	\icmlsetsymbol{equal}{*}
	\icmlsetsymbol{equaladvising}{\textdagger}
	\begin{icmlauthorlist}
		\icmlauthor{Soyeon Kim}{kaist,ineeji}
		\icmlauthor{Seongwoo Lim}{ineeji}
		\icmlauthor{Kyowoon Lee}{equaladvising,innocore}
		\icmlauthor{Jaesik Choi}{equaladvising,kaist,ineeji}
	\end{icmlauthorlist}

	\icmlaffiliation{kaist}{Kim Jaechul Graduate School of AI, Korea Advanced Institute of Science and Technology (KAIST), Daejeon, Korea}
	\icmlaffiliation{innocore}{KAIST InnoCORE LLM, KAIST, Daejeon, Korea}
	\icmlaffiliation{ineeji}{INEEJI, Seongnam, Korea}

	\icmlcorrespondingauthor{Jaesik Choi}{jaesik.choi@kaist.ac.kr}
	\icmlcorrespondingauthor{Kyowoon Lee}{leekwoon@kaist.ac.kr}

	\icmlkeywords{Input Attribution, Feature Attribution, Integrated Gradient, Data Manifold, Variational Autoencoder, Path Method, ICML}

	\vskip 0.3in
]

\printAffiliationsAndNotice{\textdagger\ Equal advising.}

\begin{abstract}
Feature attribution is central to diagnosing and trusting deep neural networks, and Integrated Gradients (IG) is widely used due to its axiomatic properties. However, IG can yield unreliable explanations when the integration path between a baseline and the input passes through regions with noisy gradients. While Guided Integrated Gradients reduces this sensitivity by adaptively updating low-gradient-magnitude features, input-space guidance still produces intermediate inputs that deviate from the data manifold. To address this limitation, we propose \emph{Manifold-Aligned Guided Integrated Gradients} (MA-GIG), which constructs attribution paths in the latent space of a pre-trained variational autoencoder. 
By decoding intermediate latent states, MA-GIG biases the path toward the learned generative manifold and reduces exposure to implausible input-space regions. 
Through qualitative and quantitative evaluations, we demonstrate that MA-GIG produces faithful explanations by aggregating gradients on path features proximal to the input. 
Consequently, our method reduces off-manifold noise and outperforms prior path-based attribution methods across multiple datasets and classifiers.
Our code is available at \magigcode
\end{abstract}
\section{Introduction}

\begin{figure*}[ht]
    \centering
    \includegraphics[width=\linewidth]{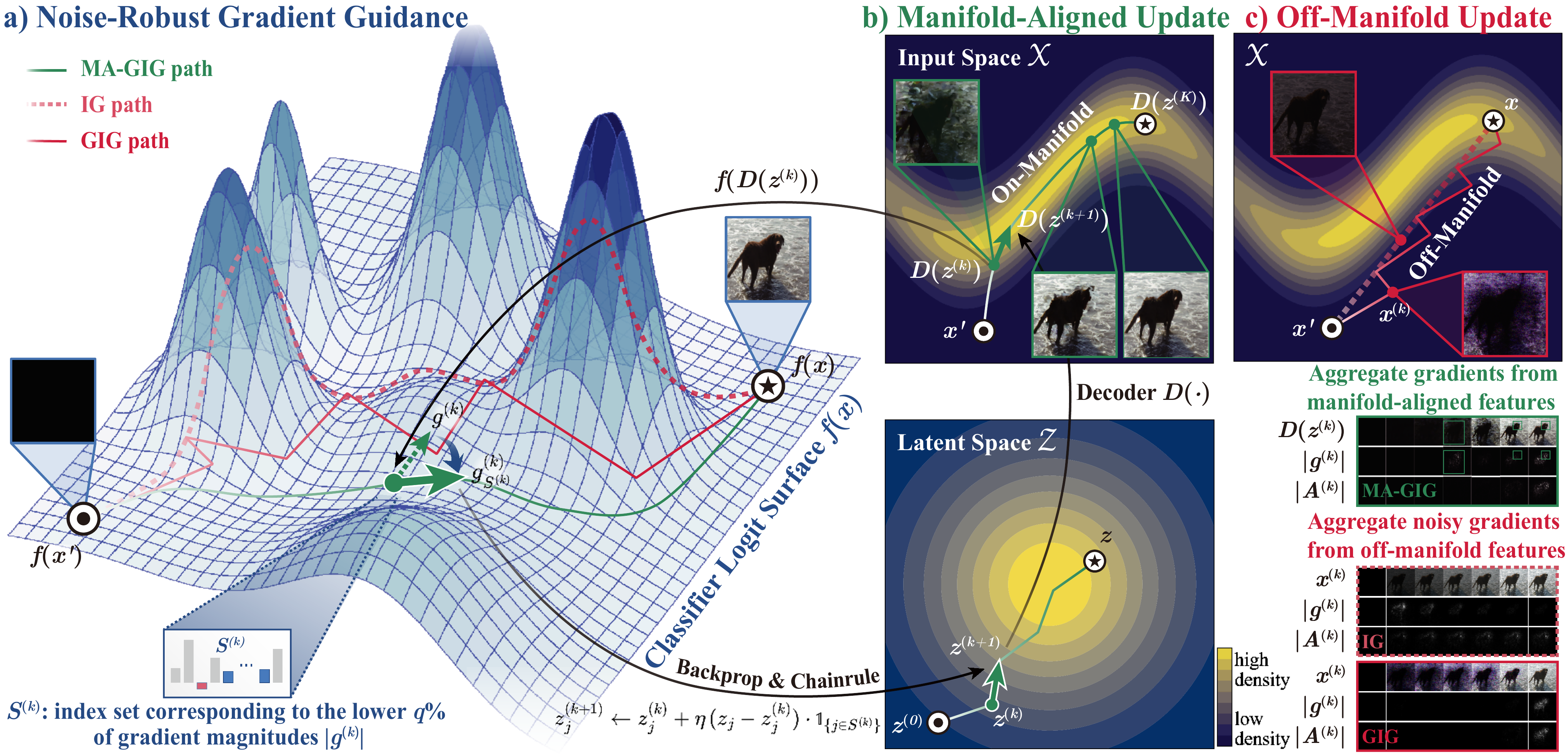}
    \caption{\textbf{Overview of Manifold-Aligned Guided Integrated Gradients (MA-GIG).} 
    \textbf{a) Noise-Robust Gradient Guidance:} The visualization compares integration paths from the baseline \protect{\raisebox{-.05cm}{\includegraphics[height=.35cm]{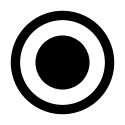}}} to the input \protect{\raisebox{-.05cm}{\includegraphics[height=.35cm]{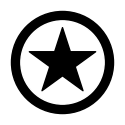}}} on the classifier's logit surface $f(x)$. The \textcolor{ma_green}{\textbf{MA-GIG path (green solid line)}} traverses noise-robust regions, avoiding the high-frequency regions traversed by the \textcolor{ma_red!50}{\textbf{linear IG path (pink dotted line)}}. This is achieved by a gradient magnitude filter based on $S^{(k)}$ that selects the lower $q$\% of gradients to suppress noise.
    \textbf{b) Manifold-Aligned Update:} The mechanism links the latent space $\mathcal{Z}$ and the learned data manifold. By updating $z^{(k)}$ selectively based on $S^{(k)}$, the trajectory is biased toward high-density regions (yellow). This encourages significant gradients to arise near the input $x$ where semantic features are progressively formed, allowing the aggregation of meaningful attributions.
    \textbf{c) Off-Manifold Update:} In contrast, unconstrained updates traverse low-density regions (blue), triggering erratic gradients even far from the input, resulting in the aggregation of noisy and dense attributions.
    }
    \label{fig:main}
\end{figure*}

Deep neural networks have achieved remarkable success across high-stakes domains, ranging from medical diagnosis \citep{chen2021deep} to autonomous navigation \citep{lee2023adaptive,lee2025scots}. However, their black-box nature often obscures the rationale behind individual predictions, precluding reliable deployment when decisions require transparent and auditable reasoning. Feature attribution methods \cite{simonyan2014saliency, lundberg2017unified, montavon2019layer} address this need by assigning an importance score to each input feature, enabling debugging, auditing, and error analysis~\citep{de2024mitigating,bilodeau2024impossibility,anders2022finding,adebayo2020debugging,schramowski2020making,rieger2020interpretations,ross2017right,lee2023refining, kim2023explainable}. Among many approaches, path-based methods~\citep{sundararajan2017axiomatic,jha2020enhanced,kapishnikov2021guided,pan2021explaining,zaher2024manifold,zhuo2024ig2} are particularly attractive because they connect attributions to principled requirements such as completeness and sensitivity \cite{sundararajan2017axiomatic}. As a result, these methods have become a standard approach for interpreting modern neural networks and for building explanation-based diagnostics.

Integrated Gradients (IG) \cite{sundararajan2017axiomatic} is a canonical path-based attribution method that computes feature contributions by integrating input gradients along a path from a baseline to the input. While IG inherits strong axiomatic guarantees, its practical reliability depends heavily on the integration path and the choice of baseline \citep{sturmfels2020visualizing, xu2020attribution, kapishnikov2021guided, jeon2023beyond, zhang2024path}. However, the straight-line path of IG can intersect with regions of high gradient or erratic fluctuations, causing spurious gradients to accumulate in the final attribution. To address this, Guided IG (GIG)~\citep{kapishnikov2021guided} generates an adaptive path that updates low-gradient-magnitude features first, effectively circumventing these noisy regions.


Despite this improvement, GIG operates directly in input space, where intermediate images along the path can still deviate substantially from the data manifold. To mitigate this, GIG utilizes straight-line \textit{anchors} to counteract excessive manifold drift; however, this mechanism reintroduces a fundamental reliance on the noisy straight-line trajectory it seeks to avoid. 
Alternatively, latent-space methods such as Enhanced IG (EIG)~\citep{jha2020enhanced} and Manifold IG (MIG)~\citep{zaher2024manifold} have been proposed to improve manifold alignment. While these approaches encourage paths through more plausible latent regions, they remain agnostic to the model's logit surface and often traverse noisy, high-curvature areas. 
Consequently, improving path construction without simultaneously controlling manifold alignment and suppressing gradient noise leaves a persistent failure mode for reliable feature attribution.

In this paper, we propose \emph{Manifold-Aligned Guided Integrated Gradients} (MA-GIG), a path-based attribution method that constructs guided integration paths in the latent space of a pre-trained variational autoencoder (VAE) \cite{kingma2013auto, higgins2017beta}. 
MA-GIG uses the decoder to bias intermediate states toward the learned data manifold, encouraging gradients to be evaluated on more plausible path samples. By reducing exposure to unsupported input-space interpolations, our approach mitigates the accumulation of spurious gradient noise. Consequently, MA-GIG yields attributions that better reflect the model's decision behavior.


Our main contributions are as follows:
\textbf{(1)} We identify manifold deviation along attribution paths as a concrete source of unreliable gradient integration, and motivate manifold-aligned path construction for faithful explanations.
\textbf{(2)} We introduce \emph{Manifold-Aligned Guided Integrated Gradients}, which leverages the latent space of generative models to bias the integration path toward the learned data manifold, so that gradients are computed on more plausible, in-distribution intermediate samples.
\textbf{(3)} We demonstrate through qualitative analyses and quantitative evaluations that MA-GIG produces more perceptually aligned and more reliable attribution maps than prior path-based attribution methods.

\section{Background}

\subsection{Problem Formulation}

We consider a classifier $f: \mathbb{R}^n \rightarrow [0,1]$ mapping input vector $x \in \mathbb{R}^n$ to a scalar target class probability score. Feature attribution quantifies the contribution of each feature $x_i$ to the prediction $f(x)$ relative to a baseline $x'$. The baseline $x'$ serves as a neutral reference, such as a zero vector representing absent information. An attribution method produces an importance map $\mathcal{A} \in \mathbb{R}^n$, where the magnitude and sign of $\mathcal{A}_i$ indicate the feature's influence on the prediction.



\subsection{Path Methods}

Path-based attribution methods define feature importance by accumulating gradients along a continuous curve $\gamma: [0, 1] \rightarrow \mathbb{R}^n$ connecting the baseline $x'$ to the input $x$. Formally, let $\gamma$ be a path function such that $\gamma(0) = x'$ and $\gamma(1) = x$. The attribution for the $i$-th feature is defined as the path integral of the gradient with respect to that feature:
\begin{equation}
\mathcal{A}_i(\gamma) = \int_{0}^{1} \frac{\partial f(\gamma(t))}{\partial \gamma_i(t)} \frac{\partial \gamma_i(t)}{\partial t} \, \mathrm{d}t.
\end{equation}
A key theoretical advantage of path methods is their satisfaction of desirable axiomatic properties, such as \emph{Completeness} \citep{sundararajan2017axiomatic}, which guarantees that the sum of attributions equals the difference in model output: $\sum_{i=1}^n \mathcal{A}_i(\gamma) = f(x) - f(x')$.

\paragraph{Integrated Gradients (IG).}
IG~\citep{sundararajan2017axiomatic} is a canonical path method that employs a straight-line path for integration, defined as $\gamma(t) = x' + t(x - x')$ for $t \in [0, 1]$. Substituting this parameterization into the general path integral yields the closed-form attribution:
\begin{equation}
\mathcal{A}^{\text{IG}}_i(x, x') = (x_i - x'_i) \times \int_{0}^{1} \frac{\partial f(x' + t(x - x'))}{\partial x_i} \, \mathrm{d}t.
\end{equation}
In practice, the integral is approximated using a Riemann sum by averaging gradients computed at $M$ discrete points uniformly spaced along the linear interpolation.

\paragraph{Guided Integrated Gradients (GIG).}
To mitigate the noise accumulated from high-variance regions, GIG~\citep{kapishnikov2021guided} constructs an adaptive path designed to minimize the accumulation of spurious gradients. The fundamental objective is to identify a path $\gamma^*$ connecting the baseline to the input that minimizes the total gradient magnitude along the trajectory:
\begin{equation}
\gamma^* = \operatorname*{argmin}_{\gamma \in \Gamma} \sum_{i=1}^{n} \int_{t=0}^1\left|\frac{\partial f(\gamma(t))}{\partial \gamma_i(t)}\frac{\partial \gamma_i(t)}{\partial t}\right| \, \mathrm{d}t,
\end{equation}
where $\Gamma$ denotes the set of valid paths from $x'$ to $x$. Since solving this global optimization is computationally intractable, GIG employs a greedy approximation that iteratively updates features with the smallest gradient magnitudes. By navigating along flatter regions of the model's prediction surface, the method mitigates noise accumulation and yields cleaner attributions than the straight-line path.

\subsection{The Manifold Hypothesis}
\label{subsec:manifold_hypothesis}

The reliability of gradient-based attribution depends on evaluating gradients within the support of the data distribution. We formalize this via the manifold hypothesis~\citep{dombrowski2023diffeomorphic, he2024manifold}, assuming that the data support $\mathcal{X}$ lies on a smooth $k$-dimensional submanifold $\mathcal{M} \subset \mathbb{R}^n$ ($k \ll n$). Under this assumption, each point $x \in \mathcal{M}$ has an associated tangent space $T_x\mathcal{M}$ that locally approximates the manifold and characterizes the directions of natural data variation. To preserve manifold alignment (to first order), a discrete update $\Delta x$ at $x$ must satisfy $\Delta x \in T_x\mathcal{M}$. Updates with components orthogonal to $T_x\mathcal{M}$ move the point off-manifold into regions where classifier behavior is undefined and gradients are unreliable. This geometric constraint motivates our manifold-aligned attribution paths. (See Appendix~\ref{app:geometric_background} for further geometric details.)





\section{Manifold-Aligned Guided IG}
\label{sec:method}


In this section, we propose \textbf{M}anifold-\textbf{A}ligned \textbf{G}uided \textbf{I}ntegrated \textbf{G}radients (\textbf{MA-GIG}), a framework that generates reliable feature attributions 
by biasing the integration path toward the learned data manifold. We first formally characterize 
the manifold deviation problem inherent to pixel-space guidance, 
then show how latent-space path construction mitigates it 
through the geometric properties of the decoder mapping.

\subsection{Manifold Deviation in Input-Space Guidance}
\label{subsec:manifold_deviation}

The reliability of feature attribution can degrade when the integration path $\gamma(t)$ leaves the support of the data manifold $\mathcal{M} \subset \mathbb{R}^n$ 
\cite{bordt2023manifold}. 
This motivates evaluating gradients on path samples that remain as plausible as possible under the data distribution. 
In this section, we demonstrate that existing methods operating in the ambient input space, specifically GIG, inherently suffer from \emph{manifold deviation}.

\paragraph{Formalization of GIG Updates.}
To mitigate the accumulation of noise from high-variance gradients, GIG employs a greedy selection strategy. At step $k$, with the current position $x^{(k)}$, the algorithm identifies a subset of features $S^{(k)}$ possessing low local gradient magnitudes:
\begin{equation}
S^{(k)} = \left\{ i \mid \left| \frac{\partial f(x^{(k)})}{\partial x_i} \right| \leq \tau^{(k)} \right\},
\end{equation}
where $\tau^{(k)}$ denotes an adaptive threshold. The algorithm subsequently updates these selected features toward the input $x$. We characterize this update as a translation vector $\Delta x^{(k)}$ defined by:
\begin{equation}
x^{(k+1)} = x^{(k)} + \Delta x^{(k)}, \quad \text{where } \Delta x^{(k)}_i = 
\begin{cases} 
\delta_i & \text{if } i \in S^{(k)} \\
0 & \text{otherwise}
\end{cases}.
\end{equation}
Here, $\delta_i$ represents the step size for feature $i$. A critical observation is that $\Delta x^{(k)}$ is sparse and strictly axis-aligned with respect to the canonical basis of $\mathbb{R}^n$.

\paragraph{Geometric Incompatibility.}
Manifold deviation stems from a geometric mismatch between the axis-aligned update $\Delta x^{(k)}$ and the local manifold curvature. We characterize the local geometry of $\mathcal{M}$ at $x^{(k)}$ via the tangent space $T_{x^{(k)}}\mathcal{M}$. To remain on the manifold (in a first-order approximation), the displacement vector must satisfy $\Delta x^{(k)} \in T_{x^{(k)}}\mathcal{M}$.

However, the tangent space of natural image data encodes complex, non-linear correlations between pixels (e.g., edges and textures). Consequently, $T_{x^{(k)}}\mathcal{M}$ is rarely aligned with the canonical coordinate axes. We decompose the update vector into tangential and orthogonal components:
\begin{equation}
\Delta x^{(k)} = \Delta x^{(k)}_{\parallel} + \Delta x^{(k)}_{\perp},
\end{equation}
where $\Delta x^{(k)}_{\parallel} = \text{proj}_{T_{x^{(k)}}\mathcal{M}}(\Delta x^{(k)})$. $\Delta x^{(k)}_{\perp}$ represents the \textit{manifold deviation error}. In input-space greedy algorithms, this error is non-zero almost everywhere. We formalize this observation in the following proposition.

\begin{restatable}[Off-Manifold Drift]{proposition}{PropGIGDeviation}
\label{prop:gig_deviation}
Let $\mathcal{M} \subset \mathbb{R}^n$ be a $C^2$ submanifold with positive reach $\tau > 0$. Consider a GIG update step $x^{(k+1)} = x^{(k)} + \Delta x^{(k)}$ starting from $x^{(k)} \in \mathcal{M}$. Let $\Delta x_{\perp}^{(k)} = \Delta x^{(k)} - \mathcal{P}_{T_{x^{(k)}}\mathcal{M}}(\Delta x^{(k)})$ be the component of the update orthogonal to the tangent space. If the step size is sufficiently small ($\|\Delta x^{(k)}\| \le \tau/2$) and the orthogonal component is dominant such that $\|\Delta x_{\perp}^{(k)}\| > \frac{1}{\tau} \|\Delta x^{(k)}\|^2$, then the updated point strictly leaves the manifold:
\begin{equation}
x^{(k+1)} \notin \mathcal{M}.
\end{equation}
\end{restatable}

\begin{sproof}
This result derives from the geometric definition of the reach $\tau$. For a manifold with reach $\tau$, the distance of any point $y \in \mathcal{M}$ from the tangent space $T_x\mathcal{M}$ is bounded by the curvature, specifically $\|\Delta x_{\perp}\| \le \frac{1}{2\tau}\|\Delta x\|^2$ \citep{federer1959curvature}. Since the GIG update $\Delta x$ is axis-aligned while $T_x\mathcal{M}$ is generally not, the orthogonal component $\|\Delta x_{\perp}\|$ is first-order ($O(\|\Delta x\|)$), whereas the manifold's curvature tolerance is second-order ($O(\|\Delta x\|^2)$). Therefore, for sufficiently small step sizes where the linear term dominates, the condition $\|\Delta x_{\perp}\| > \frac{1}{\tau}\|\Delta x\|^2$ is satisfied, implying $x^{(k+1)}$ strictly leaves the manifold. See Appendix~\ref{app:proofs} for the complete derivation.
\end{sproof}

\paragraph{Error Accumulation.}
The orthogonal component $\Delta x_{\perp}^{(k)}$ constitutes an instantaneous drift from the manifold. Over a path of length $K$, these errors accumulate:
\begin{equation}
d(x^{(K)}, \mathcal{M}) \leq \sum_{k=0}^{K-1} \|\Delta x^{(k)}_{\perp}\| + O(\kappa),
\end{equation}
where $d(\cdot, \mathcal{M})$ is the distance to the manifold and $\kappa$ represents higher-order curvature terms. Given that GIG typically requires $K = O(1/q)$ steps for a selection fraction $q$, the cumulative deviation becomes significant. 
This can lead to gradient evaluation in poorly supported regions of the input space, motivating our proposal to perform guidance within a latent space decoded toward the learned data manifold $\mathcal{M}$.


\subsection{Manifold Alignment via Latent Guidance}
\label{subsec:magig_preservation}

\def\NoNumber#1{\STATE \textcolor{gray}{#1}}

\begin{figure}[t]
\vspace{-0.3cm}
\begin{minipage}{\linewidth}
\begin{algorithm}[H]
    \caption{\textbf{Manifold-Aligned Guided IG}}
    \label{alg:magig}
    \begin{algorithmic}[1]
    \STATE \textbf{Input:} Image $x$, baseline $x'$, classifier $f$, encoder $E$, decoder $D$, steps $K$, selection fraction $q$, step size $\eta \in (0,1]$
    \STATE \textbf{Output:} Attribution map $\mathcal{A}$
    
    \NoNumber{\small{\color{gray}\texttt{// encode to latent space}}}
    \STATE $z \gets E(x)$, \quad $z' \gets E(x')$
    \STATE $z^{(0)} \gets z'$
    
    \FOR{$k = 0, \ldots, K-2$}
        \NoNumber{\small{\color{gray}\texttt{// decode to input space}}}
        \STATE $\hat{x}^{(k)} \gets D(z^{(k)})$
        \NoNumber{\small{\color{gray}\texttt{// compute latent gradients}}}
        \STATE $g^{(k)} \gets J_D(z^{(k)})^\top \nabla_x f(\hat{x}^{(k)})$
        \NoNumber{\small{\color{gray}\texttt{// select low-gradient dimensions}}}
        \STATE $\tau^{(k)} \gets q\text{-quantile of } |g^{(k)}|$
        \STATE $S^{(k)} \gets \{j : |g^{(k)}_j| \leq \tau^{(k)}\}$
        \NoNumber{\small{\color{gray}\texttt{// update selected dims toward target}}}
        \STATE $z^{(k+1)}_j \gets z^{(k)}_j + \eta \, (z_j - z^{(k)}_j) \cdot \mathbbm{1}_{\{j \in S^{(k)}\}}$ for all $j$
    \ENDFOR
    
    \NoNumber{\small{\color{gray}\texttt{// compute path integral attributions}}}
    \STATE $z^{(K)} \gets z$
    \STATE $\tilde{x}^{(0)} \gets x'$, \quad $\tilde{x}^{(K)} \gets x$
    \STATE $\tilde{x}^{(k)} \gets D(z^{(k)})$ for $k=1,\ldots,K-1$
    \STATE $\mathcal{A}_i \gets \sum_{k=0}^{K-1} \frac{\partial f(\tilde{x}^{(k)})}{\partial x_i} \cdot (\tilde{x}^{(k+1)}_i - \tilde{x}^{(k)}_i)$ \quad for all $i$
    
    \STATE \textbf{return} $\mathcal{A}$
    \end{algorithmic}
\end{algorithm}
\end{minipage}
\end{figure}

To resolve the manifold deviation problem identified in Proposition~\ref{prop:gig_deviation}, we propose constructing the guidance path within the latent space of a generative model. 
We now show that, under idealized regularity conditions on the generative mapping, this approach maps latent paths to the generator-induced manifold; in practice, it encourages more plausible intermediate samples.

\paragraph{Generative Manifold Assumption.}
Let $D: \mathcal{Z} \to \mathcal{X}$ be a decoder of a generative model (e.g., VAE) mapping from a low-dimensional latent space $\mathcal{Z} \cong \mathbb{R}^d$ to the ambient data space $\mathcal{X} \cong \mathbb{R}^n$, where $d \ll n$. We assume $D$ acts as a parametrization of the data manifold $\mathcal{M}$. Following established frameworks in generative manifold learning \citep{dombrowski2023diffeomorphic, he2024manifold}, we adopt the \emph{Perfect Autoencoder} assumption for rigorous geometric analysis.


\begin{assumption}[Perfect Autoencoder]
\label{assump:perfect_ae}
Let $\mathcal{X} \subset \mathbb{R}^n$ be the ambient space containing the data manifold $\mathcal{M} \subset \mathcal{X}$. We assume the existence of a latent space $\mathcal{Z} \cong \mathbb{R}^d$ ($d \ll n$), an encoder $E: \mathcal{X} \to \mathcal{Z}$, and a smooth decoder $D: \mathcal{Z} \to \mathcal{X}$ that satisfy three key properties. First, the model achieves exact \textit{Reconstruction} for any point on the manifold, such that $D(E(x)) = x$ for all $x \in \mathcal{M}$. Second, the decoder satisfies \textit{Surjectivity} onto the data manifold, ensuring its image coincides with the manifold ($D(\mathcal{Z}) = \mathcal{M}$). Finally, the decoder is a \textit{Smooth Immersion}, meaning its Jacobian $J_D(z) \in \mathbb{R}^{n \times d}$ has full column rank for all $z \in \mathcal{Z}$.
\end{assumption}

Under this assumption, the image of the latent space coincides with the data manifold, i.e., $D(\mathcal{Z}) = \mathcal{M}$. A crucial geometric consequence is that the Jacobian $J_D(z)$ maps vectors from the latent space $\mathbb{R}^d$ directly onto the tangent space of the manifold at $x = D(z)$. Formally, the range of the Jacobian satisfies:
\begin{equation}
\label{eq:tangent_map}
\text{Im}(J_D(z)) = T_{D(z)}\mathcal{M},
\end{equation}
where $\text{Im}$ denotes the image (or column space) of the matrix. 


\paragraph{Manifold Alignment of Latent Updates.}
In MA-GIG, we construct the integration path $\gamma_z(t)$ in the latent space $\mathcal{Z}$ by adapting the GIG strategy to the latent representations. Let $z^{(k)}$ be the current position in latent space. We first compute the gradient of the model output with respect to the latent features via the decoder's Jacobian:
\begin{equation}
\nabla_z f(D(z^{(k)})) = J_D(z^{(k)})^\top \nabla_x f(D(z^{(k)})).
\end{equation}

Then, we select a subset of latent dimensions $S_z^{(k)}$ with low-gradient-magnitudes to minimize noise accumulation:
\begin{equation}
S_z^{(k)} = \left\{ j \mid \left| \frac{\partial f(D(z^{(k)}))}{\partial z_j} \right| \le \tau_z^{(k)} \right\}.
\end{equation}

The latent update vector $\Delta z^{(k)}$ is then defined as a sparse, axis-aligned vector updating only the selected dimensions:
\begin{equation}
\Delta z^{(k)} = \sum_{j \in S_z^{(k)}} \delta_j u_j,
\end{equation}
where $\{u_j\}$ are the standard basis vectors of $\mathcal{Z}$ and $\delta_j$ is the step size. Crucially, although $\Delta z^{(k)}$ is axis-aligned in $\mathcal{Z}$, its induced update in the data space, $\Delta x^{(k)} \approx J_D(z^{(k)}) \Delta z^{(k)}$, is \emph{not} axis-aligned in $\mathbb{R}^n$. Instead, it acts as a linear combination of the Jacobian's column vectors. The complete pseudocode for this procedure is provided in Algorithm~\ref{alg:magig}.

Under Assumption~\ref{assump:perfect_ae}, since the decoder is a smooth immersion with $D(\mathcal{Z}) = \mathcal{M}$, any continuous path in the latent space maps strictly to the data manifold. Furthermore, the first-order approximation of the induced data-space update lies exactly in the tangent space:
\begin{equation}
\Delta x^{(k)} \approx J_D(z^{(k)}) \Delta z^{(k)} \in T_{D(z^{(k)})}\mathcal{M}.
\end{equation}
This result highlights the fundamental geometric advantage of MA-GIG. Recall from Proposition~\ref{prop:gig_deviation} that a sparse update in input space, $\Delta x_{\text{GIG}} = \sum \delta_i e_i$, fails to lie in $T_x\mathcal{M}$ because the canonical coordinate basis $\{e_i\}$ is generically misaligned with the manifold's curvature.

In contrast, consider a sparse update in MA-GIG, $\Delta z = \delta_j u_j$, where $u_j$ is a standard basis vector in $\mathcal{Z}$. The induced update in data space becomes:
\begin{equation}
\Delta x_{\text{MA}} \approx J_D(z) (\delta_j u_j) = \delta_j \cdot \frac{\partial D}{\partial z_j}(z).
\end{equation}
Here, the vector $\frac{\partial D}{\partial z_j}$ is precisely the $j$-th column of the Jacobian. By the smooth immersion property, this vector is by definition a tangent vector to $\mathcal{M}$. 
Thus, the decoder $D$ maps axis-aligned latent updates into correlated input-space directions that locally follow the generator-induced tangent structure and encourage global path plausibility.

\paragraph{Practical Remark on Imperfect Generators.}
While our analysis relies on the \emph{Perfect Autoencoder} assumption, in practice, we find that well-trained imperfect VAEs
, such as those in Latent Diffusion Models~\citep{rombach2022high}, 
still provide useful generative priors for latent guidance. This suggests that the decoder geometry can improve path plausibility even without exact reconstruction or a strict guarantee of true-manifold membership.

\section{Experiments}
\label{sec:experiments}

In this section, we empirically validate the effectiveness of MA-GIG. We investigate \textbf{(1)} whether MA-GIG produces quantitatively and qualitatively more perceptually aligned attributions, \textbf{(2)} whether the latent-space path construction effectively maintains manifold alignment, and \textbf{(3)} how robust the method is across different generative backbones and hyperparameter choices. Further implementation details and extended analyses are provided in Appendices~\ref{app:impl_detail} and~\ref{app:add_results}.

\subsection{Experimental Setup}

\paragraph{Datasets and Models.}

We evaluate MA-GIG on three diverse image classification benchmarks: (1) \textbf{ImageNet}~\citep{deng2009imagenet}, the standard large-scale dataset comprising 1,000 object categories; (2) \textbf{Oxford-IIIT Pet}~\citep{parkhi2012cats}, featuring 37 pet categories, selected to assess fine-grained classification performance where subtle features distinguish similar classes; and (3) \textbf{Oxford 102 Flower}~\citep{nilsback2008automated}, containing 102 flower categories with significant variations in scale, pose, and background clutter.

For the classifier backbones, we employ VGG16~\citep{simonyan2015very}, ResNet18~\citep{he2016deep}, and InceptionV1~\citep{szegedy2015going}. All models are initialized with ImageNet-pretrained weights, with those evaluated on the Oxford datasets subsequently fine-tuned on their respective training splits following standard protocols.

\paragraph{Baselines.}
To demonstrate the effectiveness of MA-GIG, we compare it against a comprehensive set of attribution methods. We include G$\times$I~\citep{shrikumar2016not} as a representative gradient-based baseline. For path-based approaches, we evaluate the foundational IG~\citep{sundararajan2017axiomatic} and distinct path construction strategies: adaptive methods designed to mitigate noise or saturation, including IG$^2$~\citep{zhuo2024ig2} which seeks steeper gradients, GIG~\citep{kapishnikov2021guided}, and AGI~\citep{pan2021explaining} which integrates along steepest ascent paths from adversarial examples; and manifold-informed latent approaches such as EIG~\citep{jha2020enhanced} and MIG~\citep{zaher2024manifold}, which employ linear interpolation and geodesic paths within the VAE latent space, respectively. 

\begin{figure*}[ht]
    \centering
    \includegraphics[width=\linewidth]{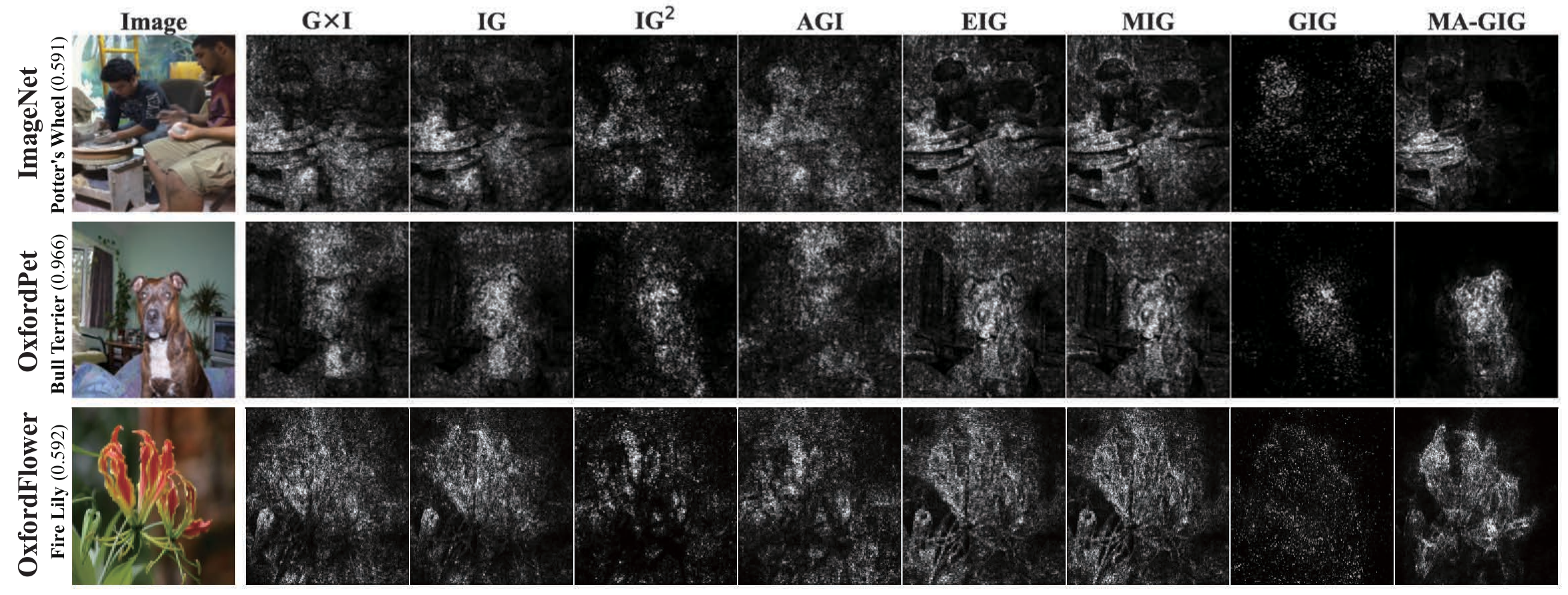}
    \caption{\textbf{Qualitative comparison of attribution maps on ImageNet (InceptionV1), Oxford-IIIT Pet (ResNet18), and Oxford 102 Flower (VGG16) against baselines.} Labels indicate predicted classes, and numbers in brackets denote prediction confidence.}
    \label{fig:qual}
\end{figure*}

\paragraph{Evaluation Metrics.}
To assess attribution faithfulness, we employ the \textbf{DiffID} score~\citep{yang2023local}.
We compute \textbf{Insertion} and \textbf{Deletion} scores~\citep{petsiuk2018rise} as the Area Under the Curve (AUC) of model confidence as pixels are progressively revealed or removed by importance. Unlike standard perturbation benchmarks prone to distribution shift~\citep{hooker2019benchmark}, DiffID mitigates biases by measuring the relative divergence between these curves rather than absolute values. This formulation neutralizes shared distributional errors, providing a reliable evaluation of feature saliency. Formal definitions are in Appendix~\ref{app:diffid_metric}.

\subsection{Qualitative Comparison}
\label{fig:qual_comp}

As shown in \Cref{fig:qual}, baselines generally produce widely dispersed attributions with significant background noise. In contrast, MA-GIG demonstrates superior visual fidelity by effectively suppressing noise and concentrating attributions on class-relevant regions. These results confirm that MA-GIG provides more faithful explanations than existing approaches. (See Appendix~\ref{app:qual} for more).

\subsection{Quantitative Evaluation}

We present the quantitative comparison of our proposed method against various baselines in \Cref{tab:attribution_comparison}, focusing on the MAR backbone. Extended evaluations with alternative VAE backbones are detailed in Appendix~\ref{app:quant}. The results demonstrate that MA-GIG consistently outperforms standard attribution methods such as IG and GIG across all evaluated benchmarks. Notably, our method achieves superior faithfulness scores in terms of DiffID and Insertion metrics on Oxford-IIIT Pet, Oxford 102 Flower, and ImageNet datasets. Furthermore, the additional results in the Appendix confirm that this performance advantage is robust across different generative priors, significantly enhancing the reliability of explanations compared to existing input-space approaches.

\begin{table*}[t]
\centering
\caption{Performance comparison of attribution methods among classifiers across datasets. \textbf{Best} and \underline{second-best} results are highlighted.}
\label{tab:attribution_comparison}
\resizebox{\textwidth}{!}{%
\begin{tabular}{llrrrrrrrrr}
\toprule
& & \multicolumn{3}{c}{\textbf{ResNet18}} & \multicolumn{3}{c}{\textbf{VGG16}} & \multicolumn{3}{c}{\textbf{InceptionV1}} \\
\cmidrule(lr){3-5} \cmidrule(lr){6-8} \cmidrule(lr){9-11}
\textbf{Data} & \textbf{Method} & \textbf{DiffID ($\uparrow$)} & \textbf{Ins ($\uparrow$)} & \textbf{Del ($\downarrow$)} & \textbf{DiffID ($\uparrow$)} & \textbf{Ins ($\uparrow$)} & \textbf{Del ($\downarrow$)} & \textbf{DiffID ($\uparrow$)} & \textbf{Ins ($\uparrow$)} & \textbf{Del ($\downarrow$)} \\
\midrule
& G $\times$ I \cite{shrikumar2016not} & 0.2384 & 0.4378 & 0.1994 & 0.4060 & 0.5174 & 0.1114 & 0.2255 & 0.3940 & 0.1685 \\
\rowcolor{palegray}
\cellcolor{white} & IG \cite{sundararajan2017axiomatic} & 0.3790 & 0.5186 & 0.1396 & 0.5255 & 0.6057 & 0.0802 & 0.3438 & 0.4748 & 0.1309 \\
\cellcolor{white} & IG$^2$ \cite{zhuo2024ig2} & \underline{0.3823} & \underline{0.5264} & 0.1441 & \textbf{0.6075} & \textbf{0.6829} & \underline{0.0754} & \underline{0.4273} & \textbf{0.5315} & \underline{0.1042} \\
\rowcolor{palegray}
\cellcolor{white} & AGI \cite{pan2021explaining} & 0.2787 & 0.4453 & 0.1667 & 0.4471 & 0.5369 & 0.0898 & 0.3381 & 0.4589 & 0.1207 \\
\cellcolor{white} & EIG \cite{jha2020enhanced} & 0.3595 & 0.4964 & \underline{0.1369} & 0.4949 & 0.5796 & 0.0847 & 0.3306 & 0.4658 & 0.1351 \\
\rowcolor{palegray}
\cellcolor{white} & MIG \cite{zaher2024manifold} & 0.3486 & 0.4889 & 0.1402 & 0.4850 & 0.5664 & 0.0814 & 0.3180 & 0.4619 & 0.1438 \\
\cellcolor{white} & GIG \cite{kapishnikov2021guided} & 0.3634 & 0.5093 & 0.1459 & 0.5556 & 0.5889 & \textbf{0.0333} & 0.3586 & 0.4880 & 0.1294 \\
\cmidrule(lr){2-11}
\rowcolor{lightgray}
\multirow{-9}{*}{\rotatebox[origin=c]{90}{\textbf{Oxford-IIIT Pet}}} \cellcolor{white} & MA-GIG & \textbf{0.4637} & \textbf{0.5886} & \textbf{0.1249} & \underline{0.5913} & \underline{0.6730} & 0.0817 & \textbf{0.4309} & \underline{0.5282} & \textbf{0.0973} \\
\midrule
\cellcolor{white} & G $\times$ I \cite{shrikumar2016not} & 0.1222 & 0.2338 & 0.1116 & 0.2784 & 0.3576 & 0.0791 & 0.2000 & 0.2813 & 0.0813 \\
\rowcolor{palegray}
\cellcolor{white} & IG \cite{sundararajan2017axiomatic} & 0.1769 & \underline{0.2740} & 0.0971 & \underline{0.3184} & \underline{0.3851} & \textbf{0.0667} & \underline{0.2551} & 0.3247 & \textbf{0.0696} \\
\cellcolor{white} & IG$^2$ \cite{zhuo2024ig2} & 0.0193 & 0.1713 & 0.1520 & 0.2224 & 0.3304 & 0.1080 & 0.0816 & 0.2053 & 0.1238 \\
\rowcolor{palegray}
\cellcolor{white} & AGI \cite{pan2021explaining} & 0.0136 & 0.1569 & 0.1433 & 0.1402 & 0.2682 & 0.1280 & 0.0787 & 0.1947 & 0.1160 \\
\cellcolor{white} & EIG \cite{jha2020enhanced} & 0.1696 & 0.2687 & 0.0991 & 0.3084 & 0.3758 & \underline{0.0673} & 0.2507 & 0.3240 & 0.0733 \\
\rowcolor{palegray}
\cellcolor{white} & MIG \cite{zaher2024manifold} & 0.1671 & 0.2707 & 0.1036 & 0.3073 & 0.3753 & 0.0680 & 0.2458 & 0.3202 & 0.0744 \\
\cellcolor{white} & GIG \cite{kapishnikov2021guided} & \underline{0.1891} & 0.2720 & \textbf{0.0829} & 0.2542 & 0.3424 & 0.0882 & \underline{0.2551} & \underline{0.3282} & \underline{0.0731} \\
\cmidrule(lr){2-11}
\rowcolor{lightgray}
\multirow{-9}{*}{\rotatebox[origin=c]{90}{\textbf{Oxford 102 Flower}}} \cellcolor{white} & MA-GIG & \textbf{0.2389} & \textbf{0.3333} & \underline{0.0944} & \textbf{0.3458} & \textbf{0.4222} & 0.0764 & \textbf{0.3131} & \textbf{0.3864} & 0.0733 \\
\midrule
\cellcolor{white} & G $\times$ I \cite{shrikumar2016not} & 0.1038 & 0.2278 & 0.1240 & 0.1882 & 0.2749 & 0.0867 & 0.1380 & 0.2776 & 0.1396 \\
\rowcolor{palegray}
\cellcolor{white} & IG \cite{sundararajan2017axiomatic} & 0.1358 & 0.2507 & 0.1149 & 0.2469 & 0.3222 & 0.0753 & 0.2020 & 0.3138 & 0.1118 \\
\cellcolor{white} & IG$^2$ \cite{zhuo2024ig2} & 0.1824 & 0.2676 & 0.0851 & \underline{0.3044} & \underline{0.3582} & \underline{0.0538} & 0.2569 & 0.3447 & \underline{0.0878} \\
\rowcolor{palegray}
\cellcolor{white} & AGI \cite{pan2021explaining} & 0.1447 & 0.2507 & 0.1060 & 0.2242 & 0.3040 & 0.0798 & 0.2067 & 0.2993 & 0.0927 \\
\cellcolor{white} & EIG \cite{jha2020enhanced} & 0.1327 & 0.2467 & 0.1140 & 0.2436 & 0.3222 & 0.0787 & 0.1942 & 0.3013 & 0.1071 \\
\rowcolor{palegray}
\cellcolor{white} & MIG \cite{zaher2024manifold} & 0.1522 & 0.2616 & 0.1093 & 0.2173 & 0.2898 & 0.0724 & 0.1600 & 0.2849 & 0.1249 \\
\cellcolor{white} & GIG \cite{kapishnikov2021guided} & \underline{0.2536} & \underline{0.3304} & \underline{0.0769} & 0.2842 & 0.3367 & \textbf{0.0524} & \underline{0.2771} & \underline{0.3667} & 0.0896 \\
\cmidrule(lr){2-11}
\rowcolor{lightgray}
\multirow{-9}{*}{\rotatebox[origin=c]{90}{\textbf{ImageNet2012}}} \cellcolor{white} & MA-GIG & \textbf{0.2707} & \textbf{0.3409} & \textbf{0.0702} & \textbf{0.3262} & \textbf{0.3809} & 0.0547 & \textbf{0.3280} & \textbf{0.4087} & \textbf{0.0807} \\
\bottomrule
\end{tabular}
}
\vspace{-0.1cm}
\end{table*}

\subsection{Manifold Alignment Analysis}
\label{sec:manifold_analysis}

To probe the perceptual plausibility of intermediate path samples, we measure the LPIPS distance~\citep{zhang2018unreasonable} from each $\gamma(\alpha)$ to the input. LPIPS is not a direct manifold metric, but provides a useful proxy for perceptual deviation along the path, since large feature-space changes often indicate visually implausible intermediate states.

\begin{figure}[!t]
    \centering
    \includegraphics[width=.95\linewidth]{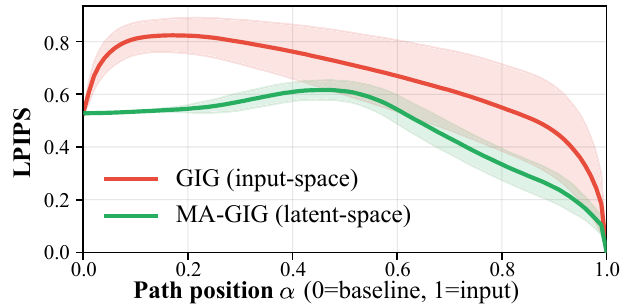}
    \caption{\textbf{LPIPS-based manifold alignment path analysis on ImageNet (ResNet18).} The LPIPS distance from each intermediate sample $\gamma(\alpha)$ to the input.}
    \label{fig:manifold_analysis}
\end{figure}

\cref{fig:manifold_analysis} contrasts the LPIPS distance profiles of GIG and MA-GIG on ImageNet using ResNet18, averaged over 100 samples. GIG exhibits higher perceptual deviation during the intermediate phase, consistent with visually unstable input-space updates. In contrast, MA-GIG maintains lower LPIPS scores throughout the trajectory, suggesting that decoded latent paths produce more perceptually coherent intermediate samples. Additional results can be found in Appendix~\ref{app:manifold}.




\begin{figure}[ht]
    \centering
    \begin{subfigure}[b]{\linewidth}
        \centering
        \includegraphics[width=1.0\linewidth]{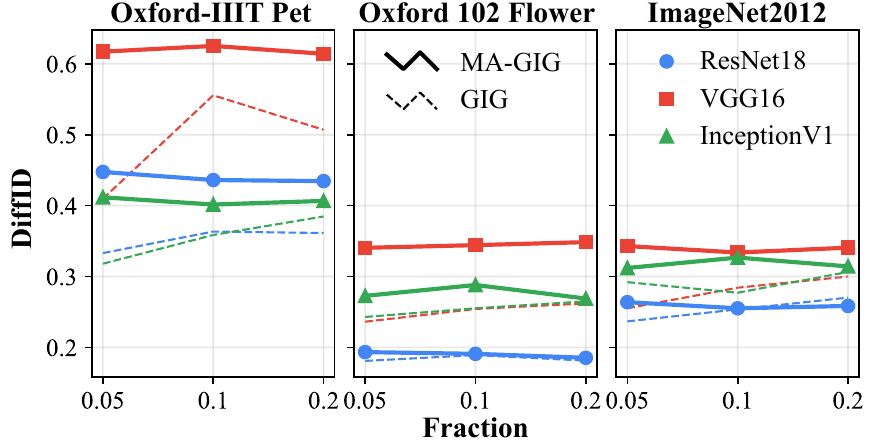}
        \caption{Effect of selection fraction}
        \label{fig:hyper_fraction}
    \end{subfigure}
    \begin{subfigure}[b]{\linewidth}
        \centering
        \includegraphics[width=\linewidth]{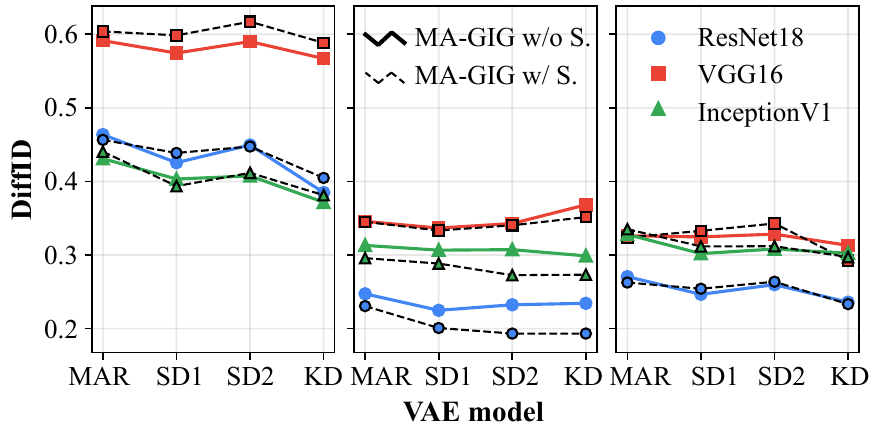}
        \caption{Performance across VAE models}
        \label{fig:hyper_vae}
    \end{subfigure}
    \caption{\textbf{Hyperparameter sensitivity and ablation analysis.} (a) The model performance remains consistent across varying feature selection fractions. (b) We compare different pre-trained VAE backbones and the effect of Spherical Linear Interpolation (Slerp). Slerp (dashed lines) shows mixed changes relative to linear interpolation (solid lines), with no consistent DiffID gain. We therefore use linear interpolation as the default path for simplicity. (S.: Slerp)}
    \label{fig:hyperparameter_search}
\end{figure}

\subsection{Effect of Selection Fraction}
We evaluate the impact of gradient guidance sparsity by varying the selection fraction $q \in \{0.05, 0.1, 0.2\}$. As shown in \Cref{fig:hyperparameter_search}(a), DiffID scores remain remarkably stable across a wide range of selection fractions $q$. Crucially, MA-GIG consistently achieves higher DiffID scores compared to the GIG baseline across all evaluated datasets and classifiers. Notably, a lower fraction of $0.05$ frequently achieves optimal performance. This stability demonstrates that our method is robust to hyperparameter choices, consistently delivering reliable attributions without requiring careful tuning of the selection fraction.
\subsection{VAE Ablation Studies}

To evaluate the generalizability of MA-GIG and its sensitivity to the generative prior, we conduct experiments using various pre-trained VAEs: MAR~\citep{li2024autoregressive}, Stable Diffusion v1-1 (SD1)~\citep{rombach2022high}, Stable Diffusion v2-1 (SD2)~\citep{rombach2022high}, and Kandinsky 2.1 (KD)~\citep{razzhigaev2023kandinsky}.



\cref{fig:hyperparameter_search}(b) demonstrates that the VAEs from MAR and SD2 generally deliver robust performance across benchmarks. Notably, the VAE from MAR, which is pre-trained specifically on ImageNet, often matches or exceeds the performance of VAEs from larger models such as SD2, despite the latter being trained on the significantly larger LAION-5B dataset~\citep{schuhmann2022laion}. This suggests that MA-GIG depends not only on generative model scale, but also on domain alignment between the VAE prior and the target data. When this alignment is poor, the decoded path may provide a weaker approximation to the data manifold and can reduce attribution quality.

\paragraph{Decoder reconstruction and attribution quality.}
We further examine whether reconstruction fidelity alone explains attribution quality. Table~\ref{tab:vae_recon_diffid} reports reconstruction MSE and DiffID on ImageNet with ResNet18 under the VAE-sweep setting. Although all VAEs have low reconstruction error, lower MSE does not necessarily imply higher attribution quality. For example, KD achieves the lowest MSE but the lowest DiffID, while MAR and SD2 yield stronger attribution scores. Across the full dataset--classifier sweep, reconstruction MSE shows only a moderate correlation with DiffID (average Pearson $r=0.406$) and Insertion AUC ($r=0.530$), suggesting that domain alignment and task-relevant latent structure matter beyond pixel-level reconstruction fidelity.

\begin{table}[ht]
    \centering
    \setlength{\tabcolsep}{3pt}
    \caption{\textbf{VAE reconstruction quality and attribution performance on ImageNet/ResNet18.}}
    \label{tab:vae_recon_diffid}
    \begin{tabular}{lcc}
        \toprule
        \textbf{VAE} & \textbf{Recon. MSE ($\downarrow$)} & \textbf{DiffID ($\uparrow$)} \\
        \midrule
        MAR \cite{li2024autoregressive} & 0.00337 & 0.2627 \\
        SD1 \cite{rombach2022high} & 0.00353 & 0.2542 \\
        SD2 \cite{rombach2022high} & 0.00275 & 0.2638 \\
        KD  \cite{razzhigaev2023kandinsky} & 0.00247 & 0.2333 \\
        \bottomrule
    \end{tabular}
\end{table}

\paragraph{Effect of Slerp.}
We investigate whether Spherical Linear Interpolation (Slerp)~\citep{shoemake1985animating}, a common alternative for latent interpolation under approximately spherical priors, improves performance. As shown in \cref{fig:hyperparameter_search}(b), Slerp yields mixed results across VAE, dataset, and classifier settings rather than a consistent improvement over standard linear interpolation. For example, it decreases the score from 0.4637 to 0.4565 on OxfordPet (ResNet18), while other settings show small gains. We therefore use linear interpolation as the default due to its simplicity and comparable performance. The specific numerical comparisons are detailed in \cref{tab:hyperparameter_search} of the Appendix.

\section{Related Work}
\label{sec:related_work}

\paragraph{Path-based Attribution Methods.}

IG~\cite{sundararajan2017axiomatic} introduced the path integral formulation for feature attribution, providing axiomatic guarantees such as completeness and sensitivity. However, the straight-line path in IG often accumulates noisy gradients along regions of high curvature on the model output surface, resulting in spurious attributions. To address this, GIG \cite{kapishnikov2021guided} adapts the integration path by greedily selecting features with the smallest gradient magnitudes, thereby avoiding high-gradient regions that contribute to noise. Blur IG~\cite{xu2020attribution} integrates along a scale-space path by progressively deblurring the input, reducing perturbation artifacts. IG$^2$~\cite{zhuo2024ig2} formulates path optimization to seek steeper gradients, avoiding saturation regions where gradient information is unreliable. SAMP~\cite{zhang2024path} formulates path selection as an optimization problem under a concentration principle, seeking paths that allocate attributions to the most salient features.

Most path-based attribution methods focus on designing more informative integration trajectories or perturbation schedules, but they rarely ask when the path itself becomes a source of attribution error. This distinction is important: completeness constrains the total attribution, but it does not prevent a path from accumulating semantically irrelevant gradients in off-manifold regions. In contrast, MA-GIG frames attribution degradation as a path-induced geometric mismatch and uses generative latent paths to reduce this mismatch.

Most relevant to our work are latent-space attribution methods. EIG employs \textit{linear} interpolation in the latent space, while MIG integrates gradients along a \textit{geodesic} shortest path. However, both methods remain largely agnostic to the model's logit surface and can still traverse noisy or high-curvature regions. In contrast, MA-GIG adaptively constructs a path that simultaneously adheres to the data manifold and avoids high-curvature regions of the prediction landscape, thereby combining geometric validity with noise suppression.


\paragraph{Mask-based and Localization-based Attribution Methods.}

Another line of work improves attribution quality through final-layer localization or explicit perturbation-mask optimization rather than path integration. Grad-CAM~\citep{selvaraju2017gradcam} localizes discriminative regions using class-specific gradients in the final convolutional layer, while I-GOS and iGOS++~\citep{qi2020visualizing,khorram2021igos++} optimize masks to identify input regions that most affect the model prediction. These methods often provide strong insertion performance, but they solve a different problem from path-based attribution: they do not characterize how gradients should be accumulated along an integration path, nor do they preserve the path-integral semantics of IG. MA-GIG is therefore complementary to these approaches, focusing specifically on reducing path-induced attribution error.

\paragraph{Manifold Constrained Optimization.}

The manifold hypothesis posits that high-dimensional data concentrates near a lower-dimensional manifold embedded in the ambient space~\citep{fefferman2016testing}. While Riemannian optimization~\citep{absil2009optimization,boumal2023introduction} offers a principled framework for explicitly known manifolds, it relies on closed-form operations such as retractions, which are intractable for manifolds implicitly defined by data distributions. 
To bridge this gap, recent works leverage generative priors to approximate manifold geometry~\citep{lee2025diverse}. 
Landing~\citep{kharitenko2025landing} introduces a link function connecting data distributions to Riemannian operations via diffusion score functions. Similarly, Manifold Preserving Guided Diffusion (MPGD)~\citep{he2024manifold} projects guidance gradients onto the tangent space using autoencoders to ensure conditional samples remain on-manifold. Diffeomorphic Counterfactuals~\citep{dombrowski2023diffeomorphic,lee2025counterfactual} leverages normalizing flows to perform gradient ascent in a coordinate system where the data manifold is better aligned, producing semantically meaningful counterfactuals. 
LoMAP~\citep{lee2025local} addresses manifold deviation in diffusion-based planning by projecting guided samples onto a locally approximated low-rank subspace derived from offline data. 

These works show that off-manifold deviation can degrade the semantic validity of generated samples, counterfactuals, or guided optimization. However, this issue has rarely been formulated as a source of error in path-integrated attribution. Aligning with this perspective, we formulate path-based attribution as a manifold-constrained problem. By operating in the VAE latent space, MA-GIG uses the decoder to project intermediate path points toward the image manifold, ensuring geometric fidelity while preserving the axiomatic structure of integrated gradients.

\section{Conclusion}
We proposed MA-GIG to resolve the persistent failure mode of off-manifold deviation in path-based feature attribution. By constructing integration paths within the latent space of generative models, our framework 
encourages gradients to be evaluated on more plausible intermediate samples aligned with the learned data manifold 
aligned with the data manifold. Our analysis demonstrates that this geometric constraint effectively suppresses off-manifold noise and enhances faithfulness by focusing on input-proximal semantic features. 
By unifying geometric fidelity with noise-robust guidance, MA-GIG provides a principled framework for reliable and transparent feature attribution essential for trustworthy AI.



\paragraph{Limitations and Future Work.}

MA-GIG addresses a specific source of error in integrated-gradient methods: the mismatch between mathematically complete path attribution and semantically meaningful attribution when the path traverses off-manifold regions. Its contribution is therefore a manifold-aligned path design principle, not a direct optimization of attribution masks or insertion scores. 
This scope also explains its empirical limitation. On ImageNet insertion, MA-GIG remains below methods such as Grad-CAM~\citep{selvaraju2017gradcam}, I-GOS, and iGOS++~\citep{qi2020visualizing,khorram2021igos++}, which rely on final-layer localization or explicitly optimize masks for the insertion objective. Future work should characterize when path-integrated attribution coincides with the desired semantic attribution, and when the gap arises. Such analysis could guide principled combinations with multi-path sampling, SmoothGrad-style averaging, or optimization-based path refinement. 

A further limitation is MA-GIG's dependence on the quality and domain alignment of the underlying VAE. If the decoder poorly reconstructs the target domain or maps latent trajectories through implausible intermediate states, the benefit of manifold alignment may weaken and attribution quality can degrade. Future work should therefore develop decoder-fidelity diagnostics and domain-adaptive generative priors for path-based attribution.

\section*{Acknowledgements}
This work was supported by the Institute for Information \& Communications Technology Planning \& Evaluation (IITP) grants funded by the Korean government (MSIT) (grant Nos. RS-2019-II190075, Artificial Intelligence Graduate School Program (KAIST); RS-2022-II220984, Development of Artificial Intelligence Technology for Personalized Plug-and-Play Explanation and Verification of Explanation; RS-2024-00457882, AI Research Hub Project), and by the InnoCORE program of the Ministry of Science and ICT (N10250156). This research was also supported by the AI Computing Infrastructure Enhancement (GPU Rental Support) User Support Program funded by MSIT (RQT-25-120227).


\section*{Impact Statement}

This paper improves the reliability of feature attribution by constraining integration paths to the data manifold, thereby producing more faithful explanations. While we identify no direct negative societal impacts, users should note that the method's reliability depends on the quality of the underlying generative model. We advise against using attributions as the sole basis for high-stakes decisions without rigorous context-specific validation.


\bibliography{example_paper}
\bibliographystyle{icml2026}

\newpage
\appendix
\onecolumn
\hypersetup{
    linkcolor=black, 
}
\section*{\huge Appendix}
\vspace{5mm}
\section*{\LARGE Table of Contents}
\vspace{-5mm}
\noindent\rule{\linewidth}{0.8pt}
\vspace{-7mm}
\begin{itemize}[label=, leftmargin=*]
    \item \hyperref[app:notation]{\textbf{Appendix A.} Notations \dotfill \pageref{app:notation}}
    \item \hyperref[app:geometric_background]{\textbf{Appendix B.} Geometric Background of Data Manifold \dotfill \pageref{app:geometric_background}}
    \item \hyperref[app:proofs]{\textbf{Appendix C.} Proofs \dotfill \pageref{app:proofs}}
    \item \hyperref[app:axioms]{\textbf{Appendix D.} Axiomatic Properties of MA-GIG \dotfill \pageref{app:axioms}}
    \begin{itemize}[label=, leftmargin=1.5em]
        \item \hyperref[app:axiomatic_proofs]{D.1. Axiomatic Proofs \dotfill \pageref{app:axiomatic_proofs}}
    \end{itemize}
    \item \hyperref[app:diffid_metric]{\textbf{Appendix E.} Evaluation Metrics \dotfill \pageref{app:diffid_metric}}
    \begin{itemize}[label=, leftmargin=1.5em]
        \item \hyperref[app:ins_del]{E.1. Insertion and Deletion Games \dotfill \pageref{app:ins_del}}
        \item \hyperref[app:diffid]{E.2. DiffID \dotfill  \pageref{app:diffid}}
    \end{itemize}
    \item \hyperref[app:impl_detail]{\textbf{Appendix F.} Implementation Details \dotfill \pageref{app:impl_detail}}
    \begin{itemize}[label=, leftmargin=1.5em]
        \item \hyperref[app:dataset_model]{F.1. Datasets and Models \dotfill \pageref{app:dataset_model}}
        \item \hyperref[app:vae]{F.2. Pre-Trained VAE Backbones for Data Manifold \dotfill \pageref{app:vae}}
        \item \hyperref[app:baselines]{F.3. Baselines \dotfill \pageref{app:baselines}}
    \end{itemize}
    \item \hyperref[app:add_results]{\textbf{Appendix G.} Additional Results \dotfill \pageref{app:add_results}}
    \begin{itemize}[label=, leftmargin=1.5em]
        \item \hyperref[app:ablation]{G.1. Detailed Hyperparameter Search and Ablation Studies \dotfill \pageref{app:ablation}}
        \item \hyperref[app:quant]{G.2. Additional Quantitative Results \dotfill \pageref{app:quant}}
        \item \hyperref[app:qual]{G.3. Additional Qualitative Results \dotfill \pageref{app:qual}}
        \item \hyperref[app:qual]{G.4. Intermediate Path Comparison \dotfill \pageref{app:path}}
    \end{itemize}
    \item \hyperref[app:manifold]{\textbf{Appendix H.} Manifold Adherence Analysis\dotfill \pageref{app:manifold}}
    \begin{itemize}[label=, leftmargin=1.5em]
        \item \hyperref[app:manifold_ext]{H.1. LPIPS-Based Path Plausibility Analysis \dotfill \pageref{app:manifold_ext}}
        \item \hyperref[app:manifold_distributional]{H.2. Distributional Diagnostics Beyond LPIPS \dotfill \pageref{app:manifold_distributional}}
        \item \hyperref[app:manifold_ext2]{H.3. Classifier-Confidence-Based Manifold Alignment Analysis \dotfill \pageref{app:manifold_ext2}}
    \end{itemize}
    \item \hyperref[app:cost]{\textbf{Appendix I.} Computational Cost \& Runtime \dotfill \pageref{app:cost}}
\end{itemize}
\vspace{-5mm}
\noindent\rule{\linewidth}{0.8pt}
\clearpage

\section{Notations}
\label{app:notation}

\begin{table}[ht]
    \centering
    \caption{\textbf{Table of notation.}}
    \label{tab:notation}
    \begin{tabular}{ll}
    \toprule 
    \textbf{Notation} & \textbf{Description} \\ 
    \midrule 
    \multicolumn{2}{l}{\textbf{Input Space \& General Attribution}} \\
    $f$ & Classifier function $f:\mathbb{R}^{n}\rightarrow[0,1]$ mapping input to a probability score. \\
    $x, x'$ & Target input image and neutral baseline reference vector in $\mathbb{R}^n$. \\
    $\mathcal{A}$ & Attribution map $\mathcal{A} \in \mathbb{R}^n$ assigning importance scores to features. \\
    $\gamma(t)$ & Continuous path function $[0,1]\rightarrow\mathbb{R}^{n}$ connecting the baseline to the input. \\
    $n$ & Dimensionality of the ambient data space $\mathcal{X} \cong \mathbb{R}^n$. \\
    \midrule
    \multicolumn{2}{l}{\textbf{Manifold \& Latent Space}} \\
    $\mathcal{M}$ & Data manifold, a smooth $k$-dimensional submanifold of $\mathbb{R}^n$. \\
    $\mathcal{X}$ & Ambient data space $\mathbb{R}^n$ containing the data manifold $\mathcal{M}$. \\ 
    $\mathcal{Z}$ & Low-dimensional latent space $\mathcal{Z} \cong \mathbb{R}^d$ where $d \ll n$. \\
    $E, D$ & Pre-trained encoder $E: \mathcal{X} \to \mathcal{Z}$ and smooth decoder $D: \mathcal{Z} \to \mathcal{X}$. \\
    $T_x \mathcal{M}$ & Tangent space of the manifold $\mathcal{M}$ at point $x$. \\
    $J_D(z)$ & Jacobian matrix of the decoder $D$ evaluated at latent position $z$. \\
    $\mathcal{P}$ & Projection operator, specifically $\mathcal{P}_{T_{x}\mathcal{M}}$ for projection onto the tangent space. \\
    $\Delta x$ & Displacement or update vector in the input space. \\
    $\Delta x_{\parallel}^{(k)}$ & Tangential component of the update vector, projected onto the tangent space $T_{x}\mathcal{M}$. \\
    $\Delta x_{\perp}^{(k)}$ & Orthogonal component of the update vector relative to the tangent space, representing manifold deviation error. \\
    \midrule
    \multicolumn{2}{l}{\textbf{MA-GIG Specific Notations}} \\
    $q$ & Selection fraction determining the sparsity of gradient guidance. \\
    $g^{(k)}$ & Latent gradient vector at step $k$, computed via the chain rule through the decoder. \\
    $S^{(k)}$ & Index set of latent dimensions with low-gradient magnitudes selected at step $k$. \\
    $g^{(k)}_{S^{(k)}}$ & Latent gradient vector at step $k$, indexed by the set of selected dimensions $S^{(k)}$. \\
    $\tau^{(k)}$ & Adaptive threshold for selecting latent features based on the $q$-quantile of gradients. \\
    $K$ & Total number of discrete integration steps along the attribution path. \\
    $\eta$ & Step size parameter for latent space updates. \\
    $\alpha, t$ & Path position or interpolation parameters in the interval $[0, 1]$. \\
    $\psi(x, \delta)$ & DiffID score measuring attribution faithfulness at perturbation level $\delta$. \\
    \bottomrule
    \end{tabular}
\end{table}

\section{Geometric Background of Data Manifold}
\label{app:geometric_background}

In this section, we provide the formal mathematical foundations for the data manifold and its metric properties, which extend the summary provided in Section~\ref{subsec:manifold_hypothesis}. The formal definition of the data manifold is defined as:

\begin{assumption}[Data Manifold]
\label{assump:appendix_manifold}
    The support $\mathcal{X}$ of the data distribution lies on a smooth $k$-dimensional submanifold $\mathcal{M} \subset \mathbb{R}^n$, where $k \ll n$. We assume $\mathcal{M}$ is a compact Riemannian manifold.
\end{assumption}

As a submanifold of the Euclidean space $\mathbb{R}^n$, $\mathcal{M}$ inherits the Riemannian metric $g$ induced by the standard Euclidean inner product. For any point $x \in \mathcal{M}$, the metric $g_x$ is defined on the tangent space $T_x\mathcal{M}$ as:

\begin{equation}
    g_x(v, w) = \langle v, w \rangle_{\mathbb{R}^n}, \quad \forall v, w \in T_x\mathcal{M}.
\end{equation}

The tangent space $T_x\mathcal{M}$ is the $k$-dimensional linear subspace of $\mathbb{R}^n$ that best approximates $\mathcal{M}$ near $x$. Any vector $v \in T_x\mathcal{M}$ represents a valid direction of variation that stays on the manifold to the first order.

To further characterize the drift described in Proposition~\ref{eq:off_manifold_drift}, we define the \textit{reach} $\tau$ of the manifold $\mathcal{M}$. The reach is the largest radius such that any point at a distance less than $\tau$ from $\mathcal{M}$ has a unique projection onto $\mathcal{M}$. This geometric property determines the maximum allowable step size for numerical path integration to avoid significant off-manifold deviation.

\section{Proofs}\label{app:proofs}

\begin{propappx}{prop:gig_deviation}{Off-Manifold Drift}
Let $\mathcal{M} \subset \mathbb{R}^n$ be a $C^2$ submanifold with positive reach $\tau > 0$. Consider a GIG update step $x^{(k+1)} = x^{(k)} + \Delta x^{(k)}$ starting from $x^{(k)} \in \mathcal{M}$. Let $\Delta x_{\perp}^{(k)} = \Delta x^{(k)} - \mathcal{P}_{T_{x^{(k)}}\mathcal{M}}(\Delta x^{(k)})$ be the component of the update orthogonal to the tangent space. If the step size is sufficiently small ($\|\Delta x^{(k)}\| \le \tau/2$) and the orthogonal component is dominant such that $\|\Delta x_{\perp}^{(k)}\| > \frac{1}{\tau} \|\Delta x^{(k)}\|^2$, then the updated point strictly leaves the manifold:
\begin{equation}
x^{(k+1)} \notin \mathcal{M}.
\end{equation}
\label{eq:off_manifold_drift}
\end{propappx}

\begin{proof}
Let $x = x^{(k)}$ and $\Delta x = \Delta x^{(k)}$. We proceed by contradiction. Assume that the updated point lies on the manifold, i.e., $y = x + \Delta x \in \mathcal{M}$. According to \citet{federer1959curvature}, if a set $\mathcal{M}$ has reach $\tau$, then for any point $x \in \mathcal{M}$, the manifold cannot intersect the interior of any open ball of radius $\tau$ that is tangent to $\mathcal{M}$ at $x$. This geometric constraint provides a bound on how far a point on the manifold can deviate from the tangent space $T_x\mathcal{M}$ over a given Euclidean distance.

A standard inequality derived from the positive reach condition states that for any $y \in \mathcal{M}$ such that $\|y - x\| < \tau$, the distance of $y$ to the tangent plane $T_x\mathcal{M}$ satisfies:
\begin{equation}
\label{eq:reach_bound}
d(y, T_x\mathcal{M}) \le \frac{1}{2\tau} \|y - x\|^2.
\end{equation}
By definition, the distance of the updated point $y = x + \Delta x$ to the tangent space is exactly the norm of the orthogonal component:
\begin{equation}
d(y, T_x\mathcal{M}) = \| \mathcal{P}_{T_x\mathcal{M}}^\perp (\Delta x) \| = \|\Delta x_{\perp}\|.
\end{equation}
Substituting this into Inequality~\eqref{eq:reach_bound}, if $y \in \mathcal{M}$, it must hold that:
\begin{equation}
\|\Delta x_{\perp}\| \le \frac{1}{2\tau} \|\Delta x\|^2.
\end{equation}

However, the proposition hypothesis states that the orthogonal component is dominant:
\begin{equation}
\|\Delta x_{\perp}\| > \frac{1}{\tau} \|\Delta x\|^2.
\end{equation}
Since $\frac{1}{\tau} > \frac{1}{2\tau}$, the hypothesis strictly violates the necessary condition for $y$ to lie on $\mathcal{M}$. Therefore, the assumption that $y \in \mathcal{M}$ must be false, proving that $x^{(k+1)} \notin \mathcal{M}$.

\end{proof}

\section{Axiomatic Properties of MA-GIG}
\label{app:axioms}

MA-GIG defines a path-based attribution method by using decoder-induced latent trajectories as the main integration path. Under the idealized autoencoder assumption, this path satisfies the standard axiomatic properties of path methods~\citep{sundararajan2017axiomatic}. With imperfect VAEs, the guarantees hold with respect to the reconstructed or endpoint-corrected path described below. We briefly outline these properties below; formal proofs are provided in Appendix~\ref{app:axiomatic_proofs}.

\paragraph{Completeness.}
The \emph{Completeness} axiom requires that the sum of feature attributions equals the difference in model output: $\sum \mathcal{A}_i = f(x) - f(x')$. For a continuous path whose endpoints are exactly $x'$ and $x$, completeness follows from the fundamental theorem of calculus. For decoder-induced paths, this is exact under Assumption~\ref{assump:perfect_ae}; with imperfect VAEs, the statement applies to reconstructed endpoints unless endpoint correction is used.

\paragraph{Sensitivity.}
MA-GIG satisfies \emph{Sensitivity(b)} (Dummy Property), ensuring that features with zero gradients along the path receive zero attribution. Under the same endpoint conditions used for Completeness, it also satisfies \emph{Sensitivity(a)}, which mandates non-zero attribution for features that functionally distinguish the input from the baseline.

\paragraph{Implementation Invariance.}
MA-GIG computes attributions using only the gradients of the composite function $f \circ D$ and the path trajectory. Since functionally equivalent networks produce identical gradients for the same inputs, the resulting path integrals are identical. Thus, MA-GIG satisfies implementation invariance.

\subsection{Axiomatic Proofs}
\label{app:axiomatic_proofs}

We provide formal justifications for the axiomatic properties of MA-GIG.

\paragraph{Notation.}
Let $f: \mathbb{R}^n \to \mathbb{R}$ be the classifier. Let $D: \mathcal{Z} \to \mathcal{X}$ be the decoder satisfying Assumption~\ref{assump:perfect_ae} (Perfect Autoencoder). The MA-GIG path is defined as $\gamma_x(t) = D(\gamma_z(t))$ for $t \in [0, 1]$, where $\gamma_z(0) = z' = E(x')$ and $\gamma_z(1) = z = E(x)$.

\begin{proposition}[Completeness of MA-GIG]
\label{prop:completeness}
Under Assumption~\ref{assump:perfect_ae}, the total attribution assigned by MA-GIG satisfies:
\begin{equation}
\sum_{i=1}^n \mathcal{A}_i = f(x) - f(x').
\end{equation}
\end{proposition}

\begin{proof}
The attribution for feature $i$ in MA-GIG is defined as the path integral of the gradient:
\begin{equation}
\mathcal{A}_i = \int_{0}^{1} \frac{\partial f(\gamma_x(t))}{\partial x_i} \frac{d \gamma_{x,i}(t)}{dt} \, dt.
\end{equation}
Summing over all features $i \in \{1, \dots, n\}$:
\begin{align}
\sum_{i=1}^n \mathcal{A}_i &= \sum_{i=1}^n \int_{0}^{1} \frac{\partial f(\gamma_x(t))}{\partial x_i} \frac{d \gamma_{x,i}(t)}{dt} \, dt \\
&= \int_{0}^{1} \left( \sum_{i=1}^n \frac{\partial f(\gamma_x(t))}{\partial x_i} \frac{d \gamma_{x,i}(t)}{dt} \right) \, dt \\
&= \int_{0}^{1} \nabla_x f(\gamma_x(t)) \cdot \dot{\gamma}_x(t) \, dt.
\end{align}
By the multivariate chain rule, the integrand is the total derivative of the composite function with respect to $t$:
\begin{equation}
\nabla_x f(\gamma_x(t)) \cdot \dot{\gamma}_x(t) = \frac{d}{dt} f(\gamma_x(t)).
\end{equation}
Applying the Fundamental Theorem of Calculus:
\begin{equation}
\int_{0}^{1} \frac{d}{dt} f(\gamma_x(t)) \, dt = f(\gamma_x(1)) - f(\gamma_x(0)).
\end{equation}
By the definition of the path endpoints and Assumption~\ref{assump:perfect_ae}:
\begin{align}
\gamma_x(1) &= D(z) = D(E(x)) = x, \\
\gamma_x(0) &= D(z') = D(E(x')) = x'.
\end{align}
Thus, the sum of attributions equals $f(x) - f(x')$.
\end{proof}


\begin{remark}[Completeness in Practice]
For imperfect VAEs where $D(E(x)) \neq x$, the decoder-induced continuous path satisfies completeness with respect to reconstructed endpoints, i.e., $f(D(E(x))) - f(D(E(x')))$. In implementation, we use an endpoint-corrected discrete path with $\tilde{x}^{(0)}=x'$, $\tilde{x}^{(K)}=x$, and $\tilde{x}^{(k)}=D(z^{(k)})$ for $1 \leq k \leq K-1$. This targets the original output difference $f(x)-f(x')$ through a polygonal path connecting the corrected endpoints and decoded intermediate samples. The resulting finite-sum attribution should therefore be interpreted as a numerical approximation to this endpoint-corrected path integral, with residual error determined by discretization and reconstruction quality.
\end{remark}

\paragraph{Completeness residual.}
We also measure the finite-sum completeness residual
$|\sum_i \mathcal{A}_i - (f(x)-f(x'))|$ with $K=200$ on ResNet18.
Endpoint correction is applied as described above. Compared to GIG,
MA-GIG reduces the residual across all datasets, indicating that the latent
path does not amplify discretization error.

\begin{table}[ht]
    \centering
    \caption{\textbf{Completeness residual reduction relative to GIG.}}
    \label{tab:completeness_residual}
    \begin{tabular}{lc}
        \toprule
        \textbf{Dataset} & \textbf{Residual Reduction vs. GIG} \\
        \midrule
        ImageNet & 9.8\% \\
        Oxford-IIIT Pet & 33.0\% \\
        Oxford 102 Flower & 30.9\% \\
        \bottomrule
    \end{tabular}
\end{table}

\begin{proposition}[Sensitivity-b / Dummy Property]
If a feature $x_i$ has no effect on the model output (i.e., $\frac{\partial f}{\partial x_i} = 0$ for all $x \in \mathbb{R}^n$), then $\mathcal{A}_i = 0$.
\end{proposition}

\begin{proof}
The attribution is given by $\mathcal{A}_i = \int_{0}^{1} \frac{\partial f(\gamma_x(t))}{\partial x_i} \dot{\gamma}_{x,i}(t) \, dt$. Since $x_i$ is a dummy variable, the partial derivative term $\frac{\partial f(\gamma_x(t))}{\partial x_i}$ is identically zero for all $t \in [0, 1]$. Consequently, the integral vanishes, and $\mathcal{A}_i = 0$.
\end{proof}

\section{Evaluation Metrics}
\label{app:diffid_metric}

The DiffID score~\citep{yang2023local} provides a unified framework for measuring attribution quality by evaluating the consistency between pixel insertion and deletion~\citep{petsiuk2018rise} processes. Its primary advantage lies in mitigating the distribution shift introduced when modifying input images, a confounding factor that often biases standard perturbation metrics.

\subsection{Insertion and Deletion Games}
\label{app:ins_del}
Let $x$ be the input image, $f$ the classifier, and $\mathcal{A}$ the attribution map. We define two complementary perturbation processes:
\begin{itemize}[leftmargin=*,itemsep=2pt]
    \item $\text{Ins}(x, \alpha)$: The \textbf{Insertion} operation creates a new image by retaining the top $\alpha$ fraction of pixels with the \textit{highest} attribution scores from $x$ and replacing the remaining $(1-\alpha)$ fraction with a baseline value.
    \item $\text{Del}(x, \delta)$: The \textbf{Deletion} operation creates a new image by removing the bottom $\delta$ fraction of pixels with the \textit{lowest} attribution scores from $x$ (replacing them with a baseline) and keeping the remaining $(1-\delta)$ fraction.
\end{itemize}
Following established practices to minimize high-frequency edge artifacts caused by pixel removal, we employ mean-imputation for the baseline values in both operations.

\subsection{DiffID}
\label{app:diffid}
The DiffID score $\psi(x, \delta)$ at a perturbation level $\delta \in [0, 1]$ is defined as the difference in model confidence between inserting the most salient features and deleting the least salient ones:
\begin{equation}
\psi(x, \delta) = f(\text{Ins}(x, 1-\delta)) - f(\text{Del}(x, \delta)).
\end{equation}
Here, $f(\cdot)$ denotes the model's output probability for the target class. This formulation leverages the insight that if an attribution map perfectly ranks feature importance, constructing an image using the top $(1-\delta)$ features ($\text{Ins}$) should yield the same model output as removing the bottom $\delta$ features ($\text{Del}$), provided there is no distribution shift. Divergence between these two quantities reflects the quality of the attribution.

The metric satisfies intuitive boundary conditions:
\begin{itemize}
    \item At $\delta=0$: $\text{Ins}(x, 1)$ is the full image $x$, and $\text{Del}(x, 0)$ is also the full image (nothing removed). Thus, $\psi(x, 0) = f(x) - f(x) = 0$.
    \item At $\delta=1$: $\text{Ins}(x, 0)$ is the baseline image $x'$, and $\text{Del}(x, 1)$ is also the baseline image (everything removed). Thus, $\psi(x, 1) = f(x') - f(x') = 0$.
\end{itemize}

\paragraph{Aggregated Score.}
The final scalar metric is obtained by integrating the DiffID curve over the perturbation range:
\begin{equation}
\text{DiffID} = \int_0^1 \psi(x, \delta) \, d\delta.
\end{equation}
A higher DiffID score indicates a more faithful attribution map, distinguishing salient features from irrelevant ones more effectively than the baseline or random assignment.

\section{Implementation Details}
\label{app:impl_detail}
In this section, we describe the experimental setup used to evaluate our proposed method. We outline the datasets, classifier models, and generative models.

\subsection{Datasets and Models}\label{app:dataset_model}
We conduct our experiments on \textbf{three} standard image classification datasets: Oxford-IIIT Pet~\cite{parkhi2012cats}, Oxford 102 Flower~\cite{nilsback2008automated}, and ImageNet2012~\cite{deng2009imagenet}.

\begin{itemize}[leftmargin=*]
\item \textbf{Oxford-IIIT Pet.} This dataset consists of 7,390 images across 37 pet breed categories. We create a validation set by randomly holding out 5\% of the images (370 images) using a fixed random seed for reproducibility.

\item \textbf{Oxford 102 Flower.} This dataset comprises 8,189 images representing 102 flower categories. We utilize the official splits provided by \texttt{torchvision}. For the DiffID evaluation, we use the first 500 images from the official validation split.

\item \textbf{ImageNet2012.} We evaluate our method on the first 500 images from the official validation set, which contains 50,000 images across 1,000 classes.
\end{itemize}

For all datasets, all images are resized to 256x256 and normalized using the standard ImageNet mean [0.485, 0.456, 0.406] and standard deviation [0.229, 0.224, 0.225].

We evaluate explanations using three widely adopted architectures: \textbf{VGG16}~\citep{simonyan2015very}, \textbf{InceptionV1}~\citep{szegedy2015going}, and \textbf{ResNet18}~\citep{he2016deep}. All models are pretrained on ImageNet and further finetuned on each dataset.

\subsection{Pre-Trained VAE Backbones for Data Manifold}\label{app:vae}
We employ four pre-trained generative models to validate our method across different latent manifold structures. All models are implemented using the \texttt{diffusers}~\cite{patrick2022diffusers} library.

\paragraph{Generative Backbones.} 

We employ the autoencoder components from four pre-trained generative models to validate our method across different latent manifold structures. Specifically, we compare the VAE from a model trained on ImageNet with VAEs from text-to-image models trained on large-scale LAION datasets (SD1, SD2, KD). All models are implemented using the \texttt{diffusers} library.

\begin{itemize}[leftmargin=*] 
    \item \textbf{MAR} (Masked Autoregressive Image Generation): A generalized autoregressive model~\cite{li2024autoregressive} trained on ImageNet-1K. Unlike conventional autoregressive approaches that rely on discrete tokens, MAR operates directly on continuous-valued tokens by employing a diffusion loss, thereby eliminating the need for vector quantization. We use the official implementation and checkpoints available at \url{https://github.com/LTH14/mar}.

    \item \textbf{SD1} (Stable Diffusion v1-1): A latent diffusion model~\cite{rombach2022high} trained on the LAION-2B (en) dataset. It utilizes a fixed CLIP ViT-L/14 text encoder and a KL-regularized autoencoder with a downsampling factor of 8 (mapping inputs to 4-channel latents). We use the \texttt{CompVis/stable-diffusion-v1-1} checkpoint.

    \item \textbf{SD2} (Stable Diffusion v2-1): An improved version of the latent diffusion model framework~\cite{rombach2022high} trained on subsets of LAION-5B with a larger text encoder (OpenCLIP-ViT/H). Unlike SD1, it employs the $v$-prediction objective during training. Comparing SD1 and SD2 allows us to analyze sensitivity within the same architectural family but under different training objectives and text embedding spaces. We use the \texttt{stabilityai/stable-diffusion-2-1} checkpoint.

    \item \textbf{KD} (Kandinsky 2.1): A text-conditional diffusion model~\cite{razzhigaev2023kandinsky} based on unCLIP architecture. Unlike the SD family, it utilizes mCLIP for embeddings and employs a MoVQGAN decoder for latent reconstruction. It is trained on large-scale datasets including LAION HighRes and fine-tuned on high-quality internal datasets. We use the \texttt{kandinsky-community/kandinsky-2-1} checkpoint.
\end{itemize}

\subsection{Baselines}\label{app:baselines}

Here we describe the specific implementations and hyperparameters used for reproducibility. Hyperparameters are chosen based on official implementations or recommendations from the respective papers, and are applied consistently across all datasets and classifiers.

\begin{itemize}
    \item \textbf{G$\times$I (Gradient $\times$ Input)}~\citep{shrikumar2016not}: A simple gradient-based baseline that multiplies input gradients with the input values.
    
    \item \textbf{Integrated Gradients (IG)}~\citep{sundararajan2017axiomatic}: We use 200 steps for the path integration with a zero baseline.
    
    \item \textbf{Guided IG (GIG)}~\citep{kapishnikov2021guided}: We use 200 integration steps with a feature selection fraction of 0.1 (10\% of features updated per step) and a zero baseline.
    
    \item \textbf{IG$^2$}~\citep{zhuo2024ig2}: We use 201 steps with a step size of 0.02.
    
    \item \textbf{Adversarial Gradient Integration (AGI)}~\citep{pan2021explaining}: We use 15 maximum iterations, 10 negative classes, and a step size of 0.05.
    
    \item \textbf{Enhanced IG (EIG)}~\citep{jha2020enhanced}: Implemented with 200 integration steps using linear interpolation in the VAE latent space with a zero baseline.
    
    \item \textbf{Manifold IG (MIG)}~\citep{zaher2024manifold}: Implemented with 200 integration steps using geodesic paths in the VAE latent space with a zero baseline. We use a learning rate of $5 \times 10^{-5}$ for geodesic optimization with 2 iterations. 
    
    \item \textbf{MA-GIG (Ours)}: We use 200 integration steps with a feature selection fraction of 0.05 (5\% of latent features updated per step) and a zero baseline. 
\end{itemize}

\section{Additional Results}
\label{app:add_results}


In this section, we provide a comprehensive evaluation to further validate the robustness and effectiveness of MA-GIG. First, we investigate the sensitivity of our method to key hyperparameters, including the feature selection fraction and the choice of VAE backbones. Second, we perform an ablation study on the geometric path construction using Spherical Linear Interpolation (Slerp).

\subsection{Detailed Hyperparameter Search and Ablation Studies}\label{app:ablation}

We provide the detailed numerical results supporting the hyperparameter sensitivity and ablation analyses presented in the main text. Table~\ref{tab:hyperparameter_search} presents the performance comparison across different feature selection fractions and VAE configurations.

\begin{table}[ht]
    \centering
    \caption{\textbf{Performance comparison across different datasets and image classifier models.} The left section analyzes feature selection fraction, while the right section presents ablation studies on Slerp and VAE from specific generative backbones. (Frac.: fraction). Best results for each setting are highlighted in \textbf{bold}.}
    \label{tab:hyperparameter_search}
    \resizebox{\textwidth}{!}{
    \begin{tabular}{ll >{\centering\arraybackslash}p{2cm} >{\centering\arraybackslash}p{2cm} >{\centering\arraybackslash}p{2cm} >{\centering\arraybackslash}p{1.5cm} >{\centering\arraybackslash}p{1.5cm} >{\centering\arraybackslash}p{1.5cm}}
        \toprule
        \multirow{2}{*}{\textbf{Dataset}} & \multirow{2}{*}{\textbf{Classifier}} & \multicolumn{3}{c}{ \textbf{Fraction Analysis} (VAE = SD2, Slerp = True)} & \multicolumn{3}{c}{\textbf{Ablation Study} (Frac. = 0.05)} \\ 
        \cmidrule(lr){3-5} \cmidrule(lr){6-8}
         & & \textbf{0.05} & \textbf{0.1} & \textbf{0.2} & \textbf{w/o Slerp} & \textbf{w/ Slerp} & \textbf{Best VAE} \\ 
        \midrule
        \multirow{3}{*}{\textbf{OxfordPet}} 
         & \textbf{ResNet18}  & \textbf{0.4474} & 0.4360 & 0.4345 & \textbf{0.4637} & 0.4565 & MAR \\
         & \textbf{VGG16}     & 0.6171 & \textbf{0.6252} & 0.6141 & \textbf{0.4495} & 0.4474 & SD2 \\
         & \textbf{Inception} & \textbf{0.4117} & 0.4015 & 0.4066 & 0.4309 & \textbf{0.4405} & MAR \\ 
        \midrule
        \multirow{3}{*}{\textbf{OxfordFlower}} 
         & \textbf{ResNet18}  & \textbf{0.1933} & 0.1909 & 0.1853 & \textbf{0.2389} & 0.2307 & MAR \\
         & \textbf{VGG16}     & 0.3404 & 0.3442 & \textbf{0.3484} & \textbf{0.3682} & 0.3516 & KD \\
         & \textbf{Inception} & 0.2727 & \textbf{0.2880} & 0.2689 & \textbf{0.3067} & 0.2884 & SD1 \\ 
        \midrule
        \multirow{3}{*}{\textbf{ImageNet}} 
         & \textbf{ResNet18}  & \textbf{0.2638} & 0.2551 & 0.2584 & 0.2598 & \textbf{0.2638} & SD2 \\
         & \textbf{VGG16}     & \textbf{0.3429} & 0.3338 & 0.3407 & 0.3284 & \textbf{0.3429} & SD2 \\
         & \textbf{Inception} & 0.3120 & \textbf{0.3267} & 0.3142 & 0.3280 & \textbf{0.3351} & MAR \\ 
        \bottomrule
    \end{tabular}
    }
\end{table}

\subsection{Additional Quantitative Results}\label{app:quant}

Table~\ref{tab:quant_full} presents an extensive comparison of attribution faithfulness metrics (DiffID, Insertion, and Deletion) across three standard benchmarks: Oxford-IIIT Pet, Oxford 102 Flower, and ImageNet.
We evaluate the performance on three different backbones—ResNet18, VGG16, and InceptionV1—to verify the generalization capability of our method.

The results demonstrate that our proposed approaches significantly surpass traditional methods such as IG, and AGI.
Notably, \textit{MA-GIG (MAR)} and the spherical interpolation variant \textit{MA-GIG Slerp} consistently yield the best performance.
For instance, on the Oxford-IIIT Pet dataset with VGG16, our method improves the DiffID score by a large margin compared to the strongest baseline.
These results confirm that leveraging generative priors for attribution path generation leads to more faithful explanations across various visual domains and model capacities.

\begin{table*}[ht]
    \centering
    \footnotesize
    \setlength{\tabcolsep}{3pt}

    \begin{minipage}[t]{0.6\textwidth}
        \centering
        \caption{\textbf{Large-scale evaluation on 5,000 ImageNet validation images with ResNet18.}}
        \label{tab:imagenet_5k}
        \begin{tabular}{@{}lccc@{}}
            \toprule
            \textbf{Method} 
            & \textbf{\shortstack{DiffID($\uparrow$)}} 
            & \textbf{\shortstack{Ins. AUC($\uparrow$)}} 
            & \textbf{\shortstack{Del. AUC($\downarrow$)}} \\
            \midrule
            G$\times$I \cite{shrikumar2016not} & 0.0978 & 0.2317 & 0.1339 \\
            IG \cite{sundararajan2017axiomatic}& 0.1538 & 0.2592 & 0.1054 \\
            IG$^2$ \cite{zhuo2024ig2} & 0.1996 & 0.2861 & 0.0866 \\
            AGI \cite{pan2021explaining} & 0.1567 & 0.2550 & 0.0982 \\
            EIG \cite{jha2020enhanced} & 0.1493 & 0.2579 & 0.1086 \\
            MIG \cite{zaher2024manifold} & 0.1631 & 0.2822 & 0.1191 \\
            GIG \cite{kapishnikov2021guided} & 0.2532 & 0.3212 & \textbf{0.0680} \\
            \rowcolor{lightgray}
            MA-GIG & \textbf{0.2670} & \textbf{0.3429} & 0.0760 \\
            \bottomrule
        \end{tabular}
    \end{minipage}
    \hfill
    \begin{minipage}[t]{0.35\textwidth}
        \centering
        \caption{\textbf{Efficient MPRT sanity check on ImageNet with ResNet18.}}
        \label{tab:mprt}
        \begin{tabular}{@{}lc@{}}
            \toprule
            \textbf{Method} & \textbf{MPRT ($\uparrow$)} \\
            \midrule
            G$\times$I \cite{shrikumar2016not} & $-0.005$ \\
            IG \cite{sundararajan2017axiomatic} & 0.073 \\
            IG$^2$ \cite{zhuo2024ig2} & $-0.063$ \\
            AGI \cite{pan2021explaining} & $-0.066$ \\
            EIG \cite{jha2020enhanced} & 0.013 \\
            MIG \cite{zaher2024manifold} & $-0.003$ \\
            GIG \cite{kapishnikov2021guided} & 0.308 \\
            \rowcolor{lightgray}
            MA-GIG & \textbf{0.625} \\
            \bottomrule
        \end{tabular}
    \end{minipage}
\end{table*}

\paragraph{Large-scale ImageNet evaluation.}
To verify that the ImageNet results are not an artifact of the 500-image subset, we additionally evaluate 5,000 ImageNet validation images using ResNet18. MA-GIG preserves the ranking observed in the main evaluation and achieves the best DiffID and Insertion AUC.

\paragraph{Model-randomization sanity check.}
We evaluate attribution dependence on model parameters using Efficient MPRT on ImageNet/ResNet18. Higher scores indicate that attributions change more consistently as model parameters are randomized. MA-GIG obtains the highest score, suggesting that its explanations are more tied to learned model behavior than to static input or decoder artifacts.

\begin{table*}[!b]
\centering
\caption{Quantitative comparison on Oxford-IIIT Pet, Oxford 102 Flower, and ImageNet datasets. Best results are highlighted in \textbf{bold}, and second-best results are \underline{underlined}.}
\label{tab:quant_full}
\resizebox{\textwidth}{!}{%
\begin{tabular}{llrrrrrrrrr}
\toprule
& & \multicolumn{3}{c}{\textbf{ResNet18}} & \multicolumn{3}{c}{\textbf{VGG16}} & \multicolumn{3}{c}{\textbf{InceptionV1}} \\
\cmidrule(lr){3-5} \cmidrule(lr){6-8} \cmidrule(lr){9-11}
\textbf{Data} & \textbf{Method} & \textbf{DiffID ($\uparrow$)} & \textbf{Ins ($\uparrow$)} & \textbf{Del ($\downarrow$)} & \textbf{DiffID ($\uparrow$)} & \textbf{Ins ($\uparrow$)} & \textbf{Del ($\downarrow$)} & \textbf{DiffID ($\uparrow$)} & \textbf{Ins ($\uparrow$)} & \textbf{Del ($\downarrow$)} \\
\midrule
& G $\times$ I \cite{shrikumar2016not} & 0.2384 & 0.4378 & 0.1994 & 0.4060 & 0.5174 & 0.1114 & 0.2255 & 0.3940 & 0.1685 \\
\rowcolor{palegray}
\cellcolor{white} & IG \cite{sundararajan2017axiomatic} & 0.3790 & 0.5186 & 0.1396 & 0.5255 & 0.6057 & 0.0802 & 0.3438 & 0.4748 & 0.1309 \\
\cellcolor{white} & IG$^2$ \cite{zhuo2024ig2} & 0.3823 & 0.5264 & 0.1441 & \underline{0.6075} & 0.6829 & 0.0754 & 0.4273 & \underline{0.5315} & \underline{0.1042} \\
\rowcolor{palegray}
\cellcolor{white} & AGI \cite{pan2021explaining} & 0.2787 & 0.4453 & 0.1667 & 0.4471 & 0.5369 & 0.0898 & 0.3381 & 0.4589 & 0.1207 \\
\cellcolor{white} & EIG \cite{jha2020enhanced} & 0.3595 & 0.4964 & 0.1369 & 0.4949 & 0.5796 & 0.0847 & 0.3306 & 0.4658 & 0.1351 \\
\rowcolor{palegray}
\cellcolor{white} & MIG \cite{zaher2024manifold} & 0.3486 & 0.4889 & 0.1402 & 0.4850 & 0.5664 & 0.0814 & 0.3180 & 0.4619 & 0.1438 \\
\cellcolor{white} & GIG \cite{kapishnikov2021guided} & 0.3634 & 0.5093 & 0.1459 & 0.5556 & 0.5889 & \textbf{0.0333} & 0.3586 & 0.4880 & 0.1294 \\
\cmidrule(lr){2-11}
\rowcolor{palegray}
\cellcolor{white} & MA-GIG (SD1) & 0.4255 & 0.5550 & 0.1294 & 0.5745 & 0.6547 & 0.0802 & 0.4033 & 0.5147 & 0.1114 \\
\cellcolor{white} & MA-GIG (SD1, w/ Slerp) & 0.4387 & 0.5652 & 0.1264 & 0.5985 & 0.6751 & 0.0766 & 0.3937 & 0.5042 & 0.1105 \\
\rowcolor{palegray}
\cellcolor{white} & MA-GIG (SD2) & 0.4495 & 0.5697 & \textbf{0.1201} & 0.5901 & 0.6601 & 0.0700 & 0.4075 & 0.5213 & 0.1138 \\
\cellcolor{white} & MA-GIG (SD2, w/ Slerp) & 0.4474 & 0.5739 & 0.1264 & \textbf{0.6171} & \textbf{0.6979} & 0.0808 & 0.4117 & 0.5252 & 0.1135 \\
\rowcolor{palegray}
\cellcolor{white} & MA-GIG (MAR) & \textbf{0.4637} & \textbf{0.5886} & 0.1249 & 0.5913 & 0.6730 & 0.0817 & \underline{0.4309} & 0.5282 & \textbf{0.0973} \\
\cellcolor{white} & MA-GIG (MAR, w/ Slerp) & \underline{0.4565} & \underline{0.5805} & \underline{0.1240} & 0.6039 & \underline{0.6832} & 0.0793 & \textbf{0.4405} & \textbf{0.5495} & 0.1090 \\
\rowcolor{palegray}
\cellcolor{white} & MA-GIG (KD) & 0.3850 & 0.5225 & 0.1375 & 0.5670 & 0.6384 & 0.0715 & 0.3718 & 0.4919 & 0.1201 \\
\multirow{-15}{*}{\rotatebox[origin=c]{90}{\textbf{Oxford-IIIT Pet}}}
\cellcolor{white} & MA-GIG (KD, w/ Slerp) & 0.4048 & 0.5333 & 0.1285 & 0.5883 & 0.6559 & \underline{0.0676} & 0.3814 & 0.4928 & 0.1114 \\
\midrule
\cellcolor{white} & G $\times$ I \cite{shrikumar2016not} & 0.1222 & 0.2338 & 0.1116 & 0.2784 & 0.3576 & 0.0791 & 0.2000 & 0.2813 & 0.0813 \\
\rowcolor{palegray}
\cellcolor{white} & IG \cite{sundararajan2017axiomatic} & 0.1769 & 0.2740 & 0.0971 & 0.3184 & 0.3851 & 0.0667 & 0.2551 & 0.3247 & 0.0696 \\
\cellcolor{white} & IG$^2$ \cite{zhuo2024ig2} & 0.0193 & 0.1713 & 0.1520 & 0.2224 & 0.3304 & 0.1080 & 0.0816 & 0.2053 & 0.1238 \\
\rowcolor{palegray}
\cellcolor{white} & AGI \cite{pan2021explaining} & 0.0136 & 0.1569 & 0.1433 & 0.1402 & 0.2682 & 0.1280 & 0.0787 & 0.1947 & 0.1160 \\
\cellcolor{white} & EIG \cite{jha2020enhanced} & 0.1696 & 0.2687 & 0.0991 & 0.3084 & 0.3758 & 0.0673 & 0.2507 & 0.3240 & 0.0733 \\
\rowcolor{palegray}
\cellcolor{white} & MIG \cite{zaher2024manifold} & 0.1671 & 0.2707 & 0.1036 & 0.3073 & 0.3753 & 0.0680 & 0.2458 & 0.3202 & 0.0744 \\
\cellcolor{white} & GIG \cite{kapishnikov2021guided} & 0.1891 & 0.2720 & \textbf{0.0829} & 0.2542 & 0.3424 & 0.0882 & 0.2551 & 0.3282 & 0.0731 \\
\cmidrule(lr){2-11}
\rowcolor{palegray}
\cellcolor{white} & MA-GIG (SD1) & 0.2249 & 0.3240 & 0.0991 & 0.3367 & 0.4109 & 0.0742 & 0.3067 & 0.3800 & 0.0733 \\
\cellcolor{white} & MA-GIG (SD1, w/ Slerp) & 0.2009 & 0.3064 & 0.1056 & 0.3333 & 0.4131 & 0.0798 & 0.2884 & 0.3638 & 0.0753 \\
\rowcolor{palegray}
\cellcolor{white} & MA-GIG (SD2) & 0.2324 & 0.3253 & 0.0929 & 0.3429 & 0.4191 & 0.0762 & \underline{0.3073} & 0.3756 & 0.0682 \\
\cellcolor{white} & MA-GIG (SD2, w/ Slerp) & 0.1933 & 0.2998 & 0.1064 & 0.3404 & 0.4196 & 0.0791 & 0.2727 & 0.3451 & 0.0724 \\
\rowcolor{palegray}
\cellcolor{white} & MA-GIG (MAR) & \textbf{0.2389} & \textbf{0.3333} & 0.0944 & 0.3458 & 0.4222 & 0.0764 & \textbf{0.3131} & \textbf{0.3864} & 0.0733 \\
\cellcolor{white} & MA-GIG (MAR, w/ Slerp) & 0.2307 & \underline{0.3269} & 0.0962 & 0.3451 & \underline{0.4244} & 0.0793 & 0.2958 & 0.3720 & 0.0762 \\
\rowcolor{palegray}
\cellcolor{white} & MA-GIG (KD) & \underline{0.2344} & 0.3222 & \underline{0.0878} & \textbf{0.3682} & \textbf{0.4269} & \textbf{0.0587} & 0.2989 & 0.3616 & \textbf{0.0627} \\
\multirow{-15}{*}{\rotatebox[origin=c]{90}{\textbf{Oxford 102 Flower}}} \cellcolor{white} & MA-GIG (KD, w/ Slerp) & 0.1933 & 0.2920 & 0.0987 & \underline{0.3516} & 0.4140 & \underline{0.0624} & 0.2733 & 0.3402 & \underline{0.0669} \\
\midrule
\cellcolor{white} & G $\times$ I \cite{shrikumar2016not} & 0.1038 & 0.2278 & 0.1240 & 0.1882 & 0.2749 & 0.0867 & 0.1380 & 0.2776 & 0.1396 \\
\rowcolor{palegray}
\cellcolor{white} & IG \cite{sundararajan2017axiomatic} & 0.1358 & 0.2507 & 0.1149 & 0.2469 & 0.3222 & 0.0753 & 0.2020 & 0.3138 & 0.1118 \\
\cellcolor{white} & IG$^2$ \cite{zhuo2024ig2} & 0.1824 & 0.2676 & 0.0851 & 0.3044 & 0.3582 & \underline{0.0538} & 0.2569 & 0.3447 & 0.0878 \\
\rowcolor{palegray}
\cellcolor{white} & AGI \cite{pan2021explaining} & 0.1447 & 0.2507 & 0.1060 & 0.2242 & 0.3040 & 0.0798 & 0.2067 & 0.2993 & 0.0927 \\
\cellcolor{white} & EIG \cite{jha2020enhanced} & 0.1327 & 0.2467 & 0.1140 & 0.2436 & 0.3222 & 0.0787 & 0.1942 & 0.3013 & 0.1071 \\
\rowcolor{palegray}
\cellcolor{white} & MIG \cite{zaher2024manifold} & 0.1522 & 0.2616 & 0.1093 & 0.2173 & 0.2898 & 0.0724 & 0.1600 & 0.2849 & 0.1249 \\
\cellcolor{white} & GIG \cite{kapishnikov2021guided} & 0.2536 & 0.3304 & 0.0769 & 0.2842 & 0.3367 & \textbf{0.0524} & 0.2771 & 0.3667 & 0.0896 \\
\cmidrule(lr){2-11}
\rowcolor{palegray}
\cellcolor{white} & MA-GIG (SD1) & 0.2467 & 0.3333 & 0.0867 & 0.3247 & 0.3869 & 0.0622 & 0.3020 & 0.3871 & 0.0851 \\
\cellcolor{white} & MA-GIG (SD1, w/ Slerp) & 0.2542 & 0.3382 & 0.0840 & \underline{0.3329} & 0.3933 & 0.0604 & 0.3118 & 0.3893 & 0.0776 \\
\rowcolor{palegray}
\cellcolor{white} & MA-GIG (SD2) & 0.2598 & 0.3409 & 0.0811 & 0.3284 & 0.3900 & 0.0616 & 0.3082 & 0.3838 & \textbf{0.0756} \\
\cellcolor{white} & MA-GIG (SD2, w/ Slerp) & 0.2638 & \textbf{0.3478} & 0.0840 & \textbf{0.3429} & \textbf{0.3982} & 0.0553 & 0.3120 & 0.3887 & \underline{0.0767} \\
\rowcolor{palegray}
\cellcolor{white} & MA-GIG (MAR) & \textbf{0.2707} & \underline{0.3409} & \textbf{0.0702} & 0.3262 & 0.3809 & 0.0547 & \underline{0.3280} & \underline{0.4087} & 0.0807 \\
\cellcolor{white} & MA-GIG (MAR, w/ Slerp) & 0.2627 & 0.3344 & \underline{0.0718} & 0.3238 & 0.3829 & 0.0591 & \textbf{0.3351} & \textbf{0.4129} & 0.0778 \\
\rowcolor{palegray}
\cellcolor{white} & MA-GIG (KD) & 0.2360 & 0.3107 & 0.0747 & 0.3131 & 0.3678 & 0.0547 & 0.3024 & 0.3833 & 0.0809 \\
\multirow{-15}{*}{\rotatebox[origin=c]{90}{\textbf{ImageNet2012}}}
\cellcolor{white} & MA-GIG (KD, w/ Slerp) & 0.2333 & 0.3076 & 0.0742 & 0.2927 & 0.3511 & 0.0584 & 0.2973 & 0.3809 & 0.0836 \\
\bottomrule
\end{tabular}
}
\vspace{-0.1cm}
\end{table*}

\subsection{Additional Qualitative Results} \label{app:qual}

As shown in Figures~\ref{fig:qual_imagenet}, \ref{fig:qual_oxfordflower}, and \ref{fig:qual_oxfordpet}, we provide additional qualitative comparisons of MA-GIG against baselines: G$\times$I~\citep{shrikumar2016not}, IG~\citep{sundararajan2017axiomatic}, IG$^{2}$~\citep{zhuo2024ig2}, AGI~\citep{pan2021explaining}, EIG~\citep{jha2020enhanced}, MIG~\citep{zaher2024manifold}, and GIG~\citep{kapishnikov2021guided}. Compared to these baselines, MA-GIG generates sparse, noise-reduced attributions and more accurately localizes class-relevant regions.

\begin{figure}[ht]
    \centering
    \includegraphics[width=\linewidth]{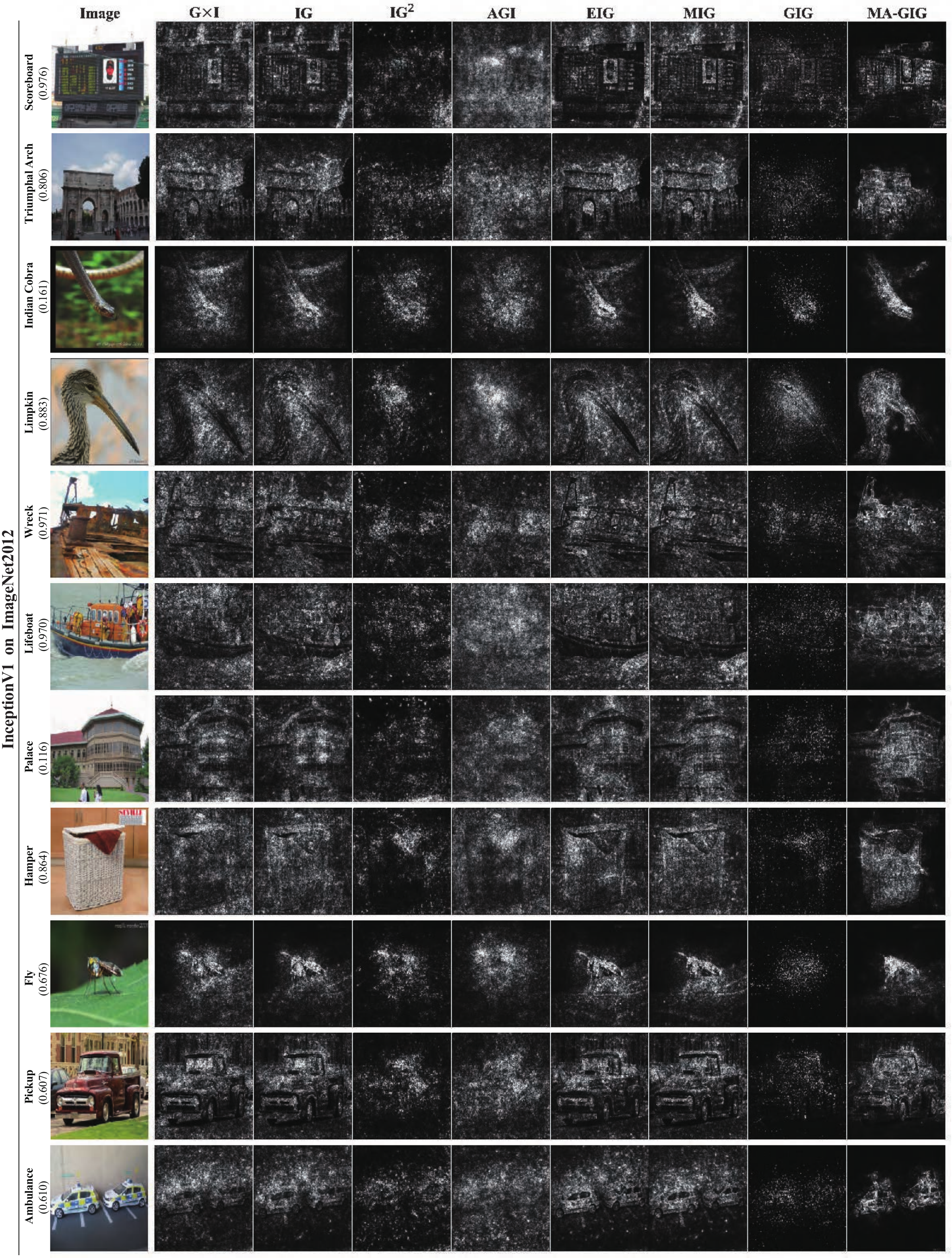}
    \caption{\textbf{Qualitative comparison against baselines on ImageNet (InceptionV1).} Left labels indicate the predicted class, and numbers in brackets denote confidence.}
    \label{fig:qual_imagenet}
\end{figure}

\begin{figure}[ht]
    \centering
    \includegraphics[width=\linewidth]{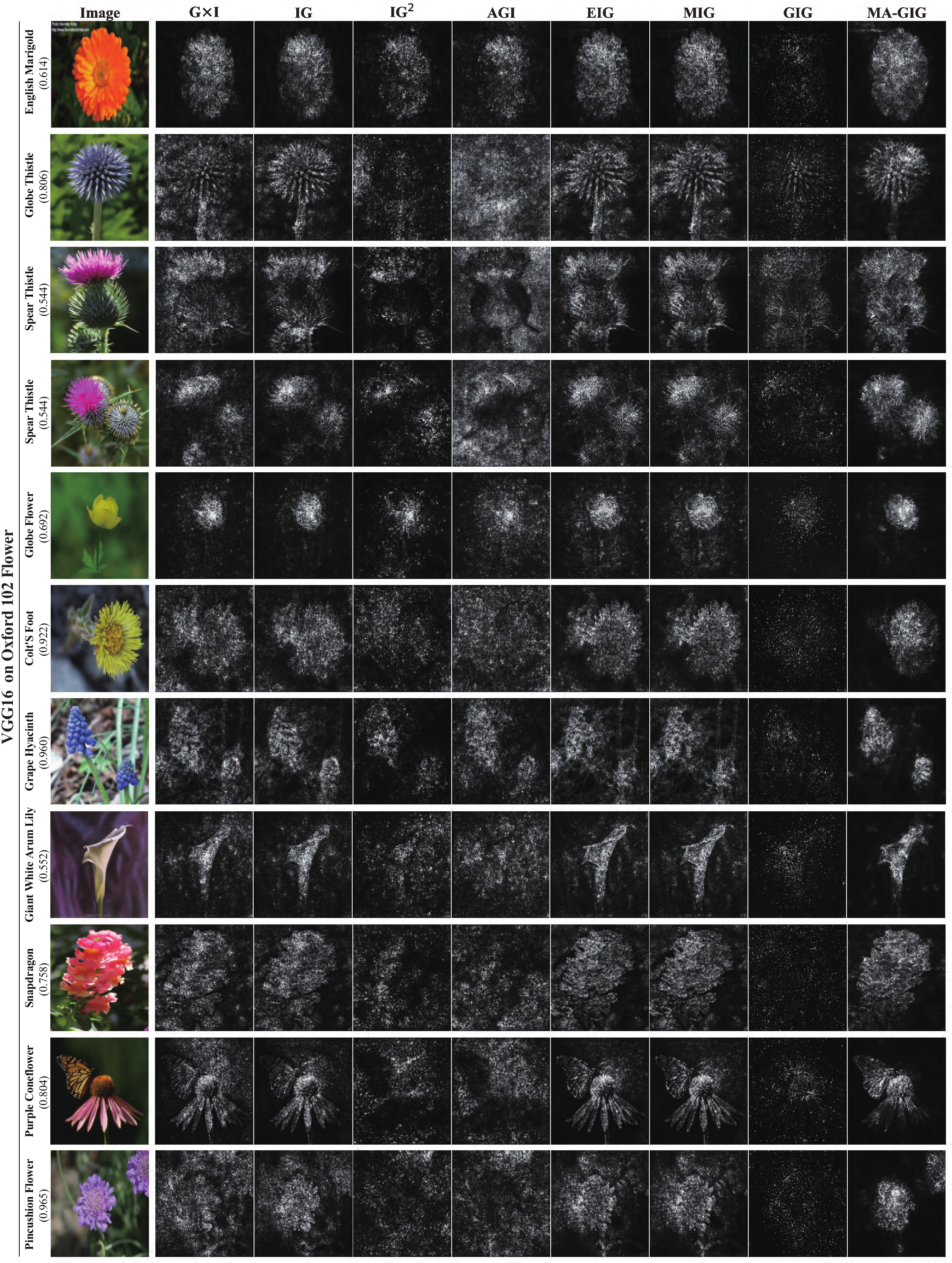}
    \caption{\textbf{Qualitative comparison on Oxford 102 Flower (VGG16).} Labels on the left indicate predicted classes, and numbers in brackets denote prediction confidence.}
    \label{fig:qual_oxfordflower}
\end{figure}

\begin{figure}[ht]
    \centering
    \includegraphics[width=\linewidth]{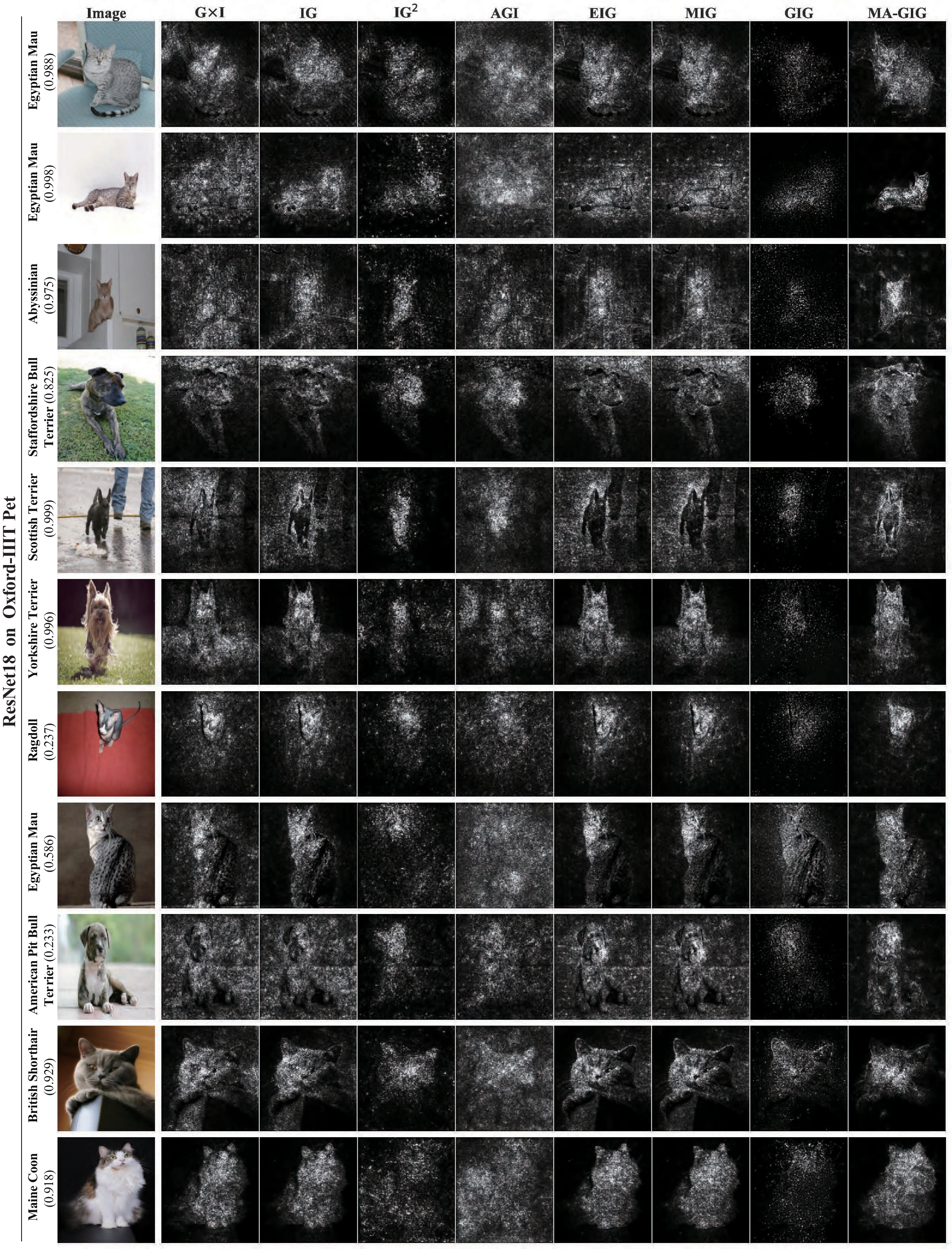}
    \caption{\textbf{Qualitative comparison on Oxford-IIIT Pet (ResNet18).} Labels on the left indicate predicted classes, and numbers in brackets denote prediction confidence.}
    \label{fig:qual_oxfordpet}
\end{figure}

\clearpage

\subsection{Intermediate Path Comparison} \label{app:path}
As an intermediate path feature analysis, we present extensive qualitative comparisons of our MA-GIG against baseline path-based methods—IG~\citep{sundararajan2017axiomatic}, EIG~\citep{jha2020enhanced}, MIG~\citep{zaher2024manifold}, and GIG~\citep{kapishnikov2021guided}—in Figures~\ref{fig:add_qual_1}, \ref{fig:add_qual_2}, \ref{fig:add_qual_3}, and \ref{fig:add_qual_4}. These visualizations cover various classifiers on the ImageNet2012, Oxford-IIIT Pet, and Oxford 102 Flower datasets.

The visualized intermediate steps (second and third rows in each panel) provide empirical evidence for the reliability of our framework. First, regarding gradient behavior, pixel-space guidance methods like GIG often suffer from manifold deviation, leading to gradients evaluated on out-of-distribution data. This results in noisy attributions.

In contrast, MA-GIG constructs the integration path from the baseline ($x'$) toward the input ($x$) within the latent manifold. This trajectory significantly reduces attribution noise in regions irrelevant to the predicted class. This robustness stems from our robust path-finding strategy: by prioritizing the restoration of latent features with the smallest gradients first (lower 5\%), the method carefully navigates the manifold, effectively avoiding high-frequency noise inherent in less robust directions.

Consequently, as shown in Figures~\ref{fig:add_qual_1}, \ref{fig:add_qual_2}, and \ref{fig:add_qual_3}, the manifold-aligned trajectory results in a distinct accumulation of gradients. As visually apparent in the \textit{Gradients $\times \Delta$ Features} rows, MA-GIG aggregates significant gradients primarily in the path segments proximal to the input image (i.e., the final stages of the path). This indicates that the method captures high-level semantic features precisely when the salient object is fully formed. In contrast, baseline methods often trigger large gradients near the black baseline, implying a sensitivity to off-manifold artifacts or simple pixel intensity.

The visualization of path features further highlights the structural integrity of our approach. While some intermediate path features may appear artifact-like, they crucially preserve the fundamental structure of the salient object. This phenomenon arises from the algorithmic design of our guided integration, which prioritizes the recovery of less important (low-gradient) background features first, while deferring the incorporation of the most important (high-gradient) features until the path approaches the target input. By delaying the inclusion of high-gradient features to the end of the path, MA-GIG ensures that the most critical structural details are aggregated at the most semantically relevant stage, resulting in the superior localization of salient objects evidenced in the attribution results.

Furthermore, our results on the OxfordFlower dataset demonstrate that MA-GIG faithfully reveals spurious correlations~\citep{lapuschkin2019unmasking}. In cases where watermarks frequently co-occur with specific classes (e.g., sunflower, sweet pea, bird of paradise, and colt's foot), MA-GIG explicitly highlights these artifacts in the attribution maps. This confirms that our method not only localizes valid semantic features but also transparently exposes model biases derived from dataset artifacts.

\begin{figure}[!b]
    \centering
    \includegraphics[width=0.48\linewidth]{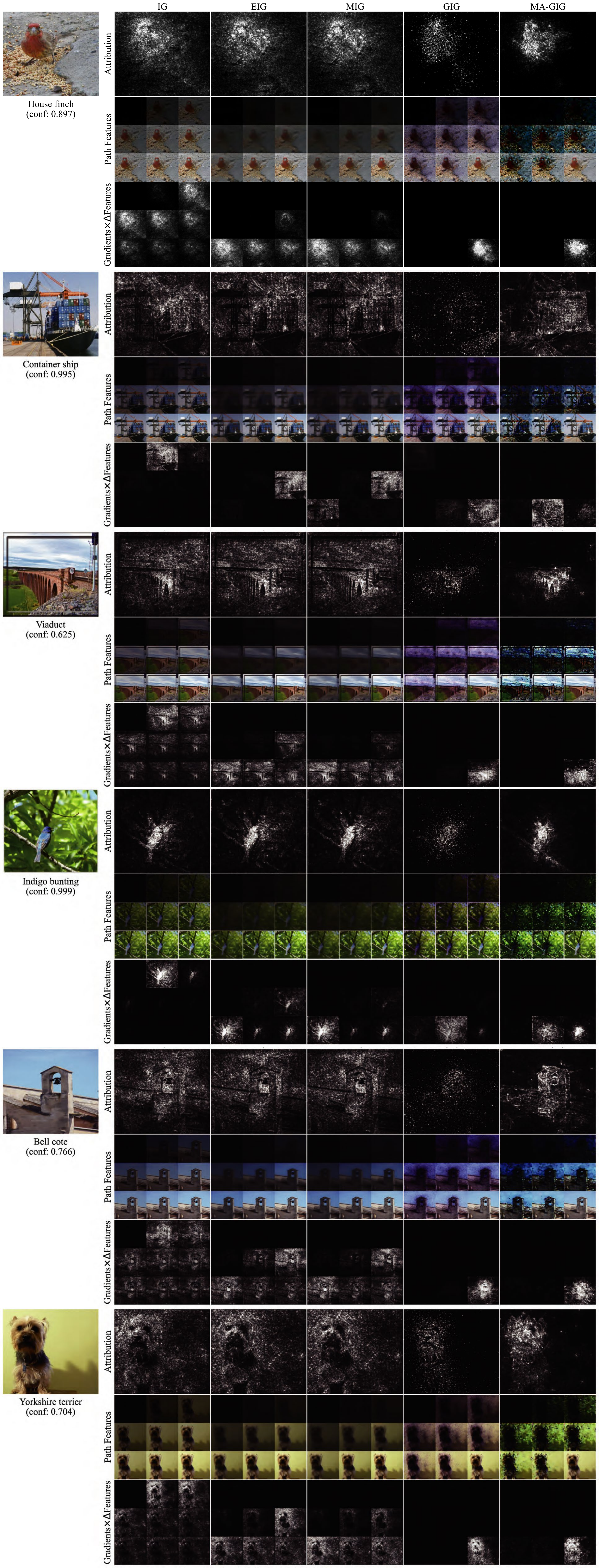}
    \hfill
    \includegraphics[width=0.48\linewidth]{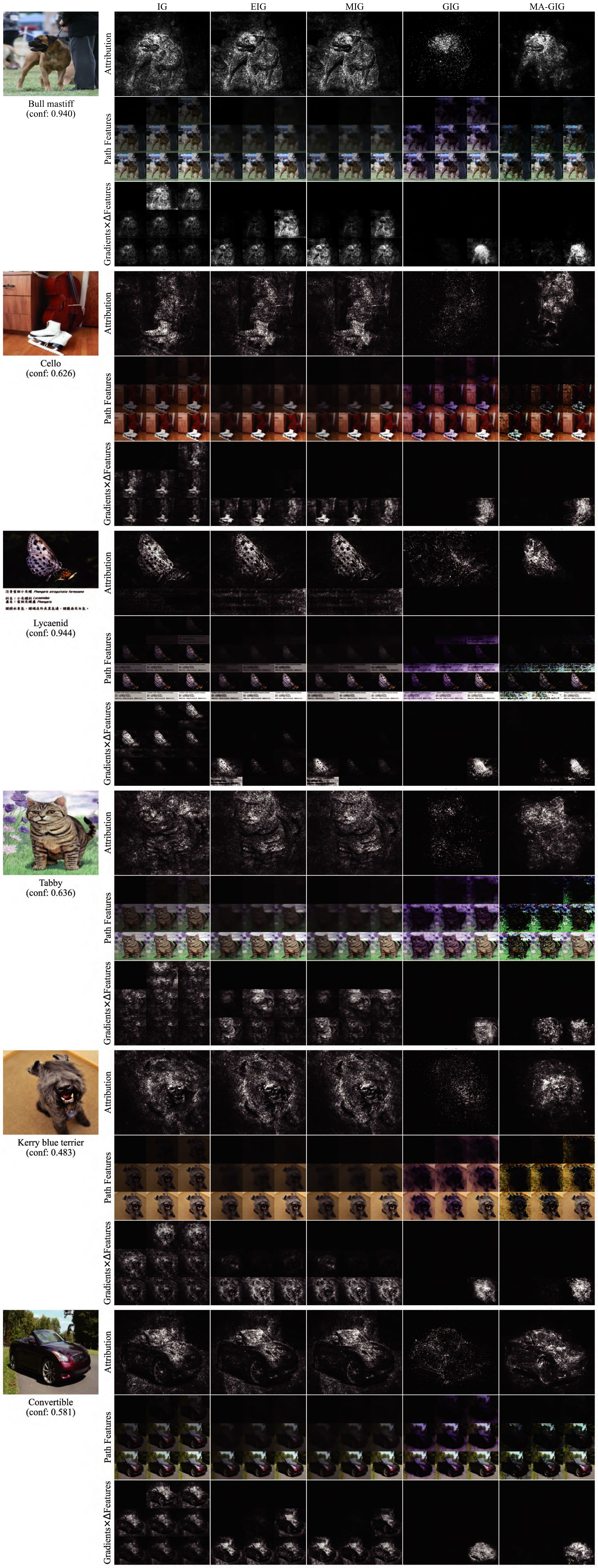}
    \caption{\textbf{Qualitative comparison on ImageNet2012 (ResNet18 \& VGG16).} The left and right panels display results for the ResNet18 and VGG16 classifiers, respectively. For each example, the top row presents the attribution maps of IG, EIG, MIG, GIG, and MA-GIG. The second and third rows visualize the evolution of path features and their corresponding gradients, sampled at nine equally spaced intervals along the integration path, demonstrating how MA-GIG aggregates relevant attributions. (Conf.: Confidence)}
    \label{fig:add_qual_1}
\end{figure}

\begin{figure}[!b]
    \centering
    \includegraphics[width=0.48\linewidth]{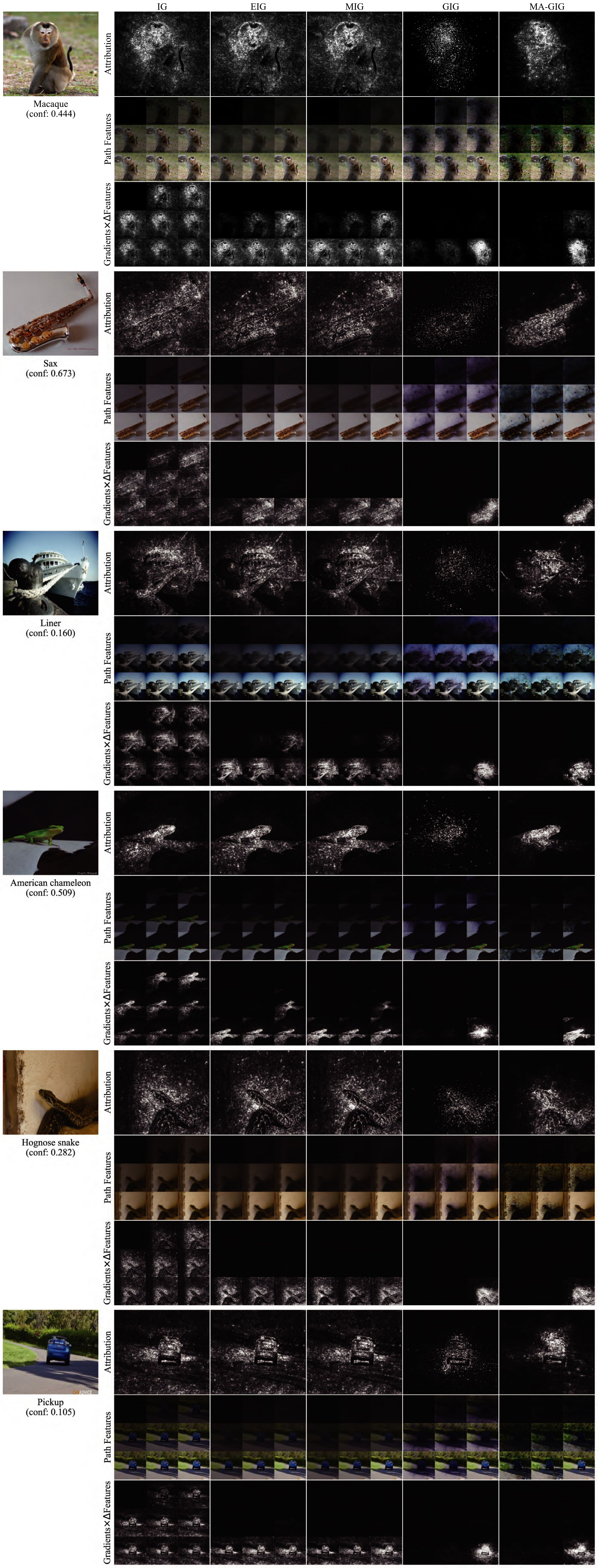}
    \hfill
    \includegraphics[width=0.48\linewidth]{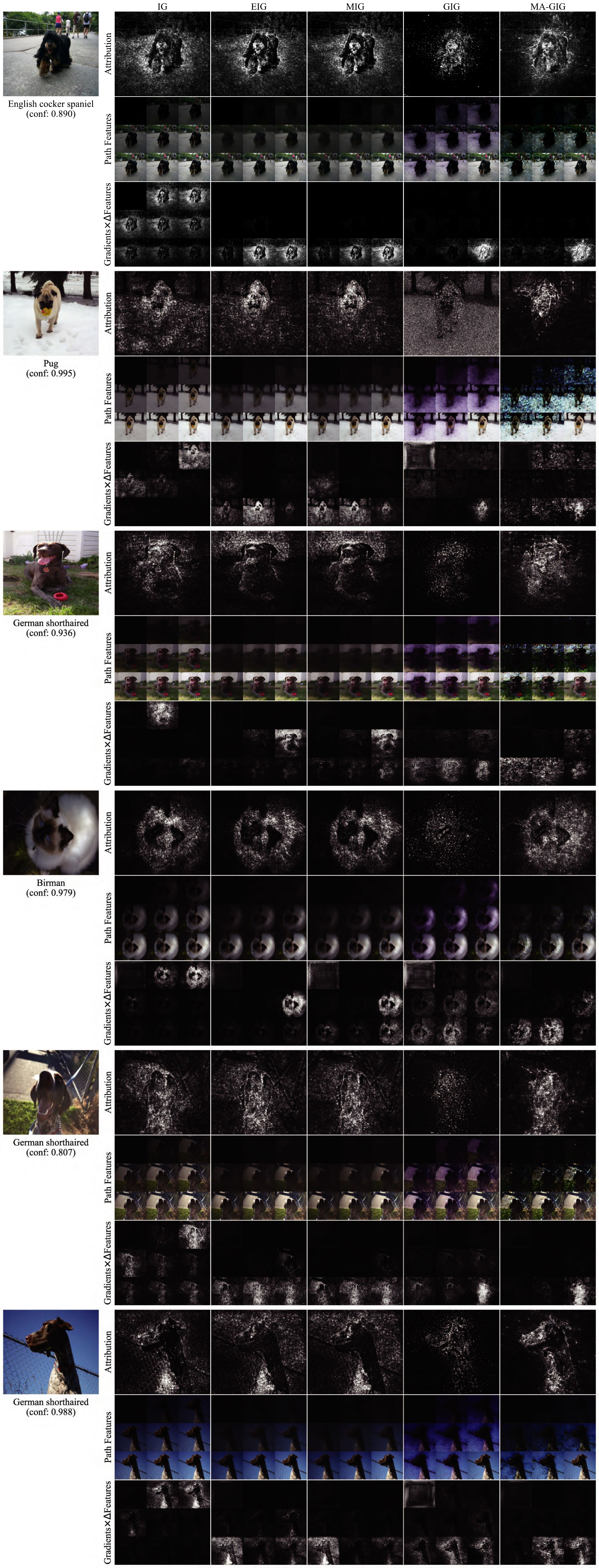}
    \caption{\textbf{Qualitative comparison on ImageNet2012 (InceptionV1) and Oxford-IIIT Pet (ResNet18).} For each example, the top row presents the attribution maps of IG, EIG, MIG, GIG, and MA-GIG. The second and third rows visualize the evolution of path features and their corresponding gradients, sampled at nine equally spaced intervals along the integration path, demonstrating how MA-GIG aggregates relevant attributions. (Conf.: Confidence)}
    \label{fig:add_qual_2}
\end{figure}

\begin{figure}[!b]
    \centering
    \includegraphics[width=0.48\linewidth]{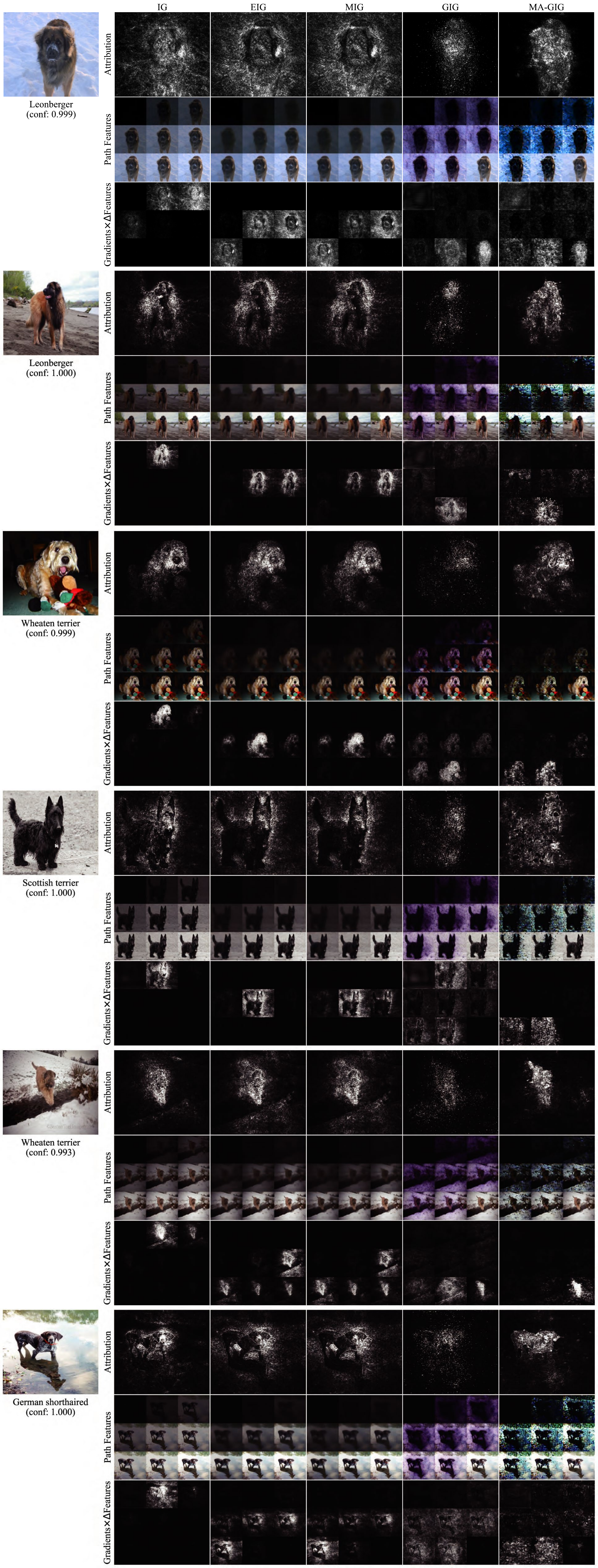}
    \hfill
    \includegraphics[width=0.48\linewidth]{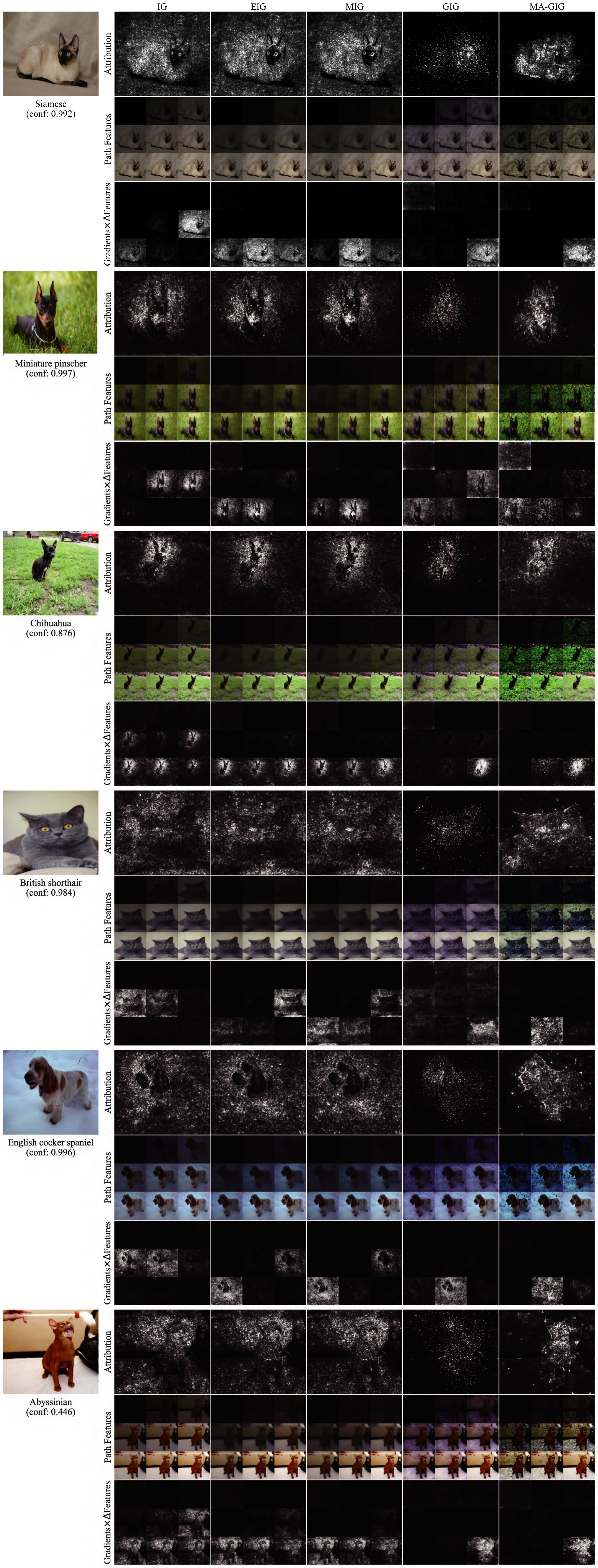}
    \caption{\textbf{Qualitative comparison on Oxford-IIIT Pet (VGG16 \& InceptionV1).} For each example, the top row presents the attribution maps of IG, EIG, MIG, GIG, and MA-GIG. The second and third rows visualize the evolution of path features and their corresponding gradients, sampled at nine equally spaced intervals along the integration path, demonstrating how MA-GIG aggregates relevant attributions. (Conf.: Confidence)}
    \label{fig:add_qual_3}
\end{figure}

\begin{figure}[!b]
    \centering
    \includegraphics[width=0.48\linewidth]{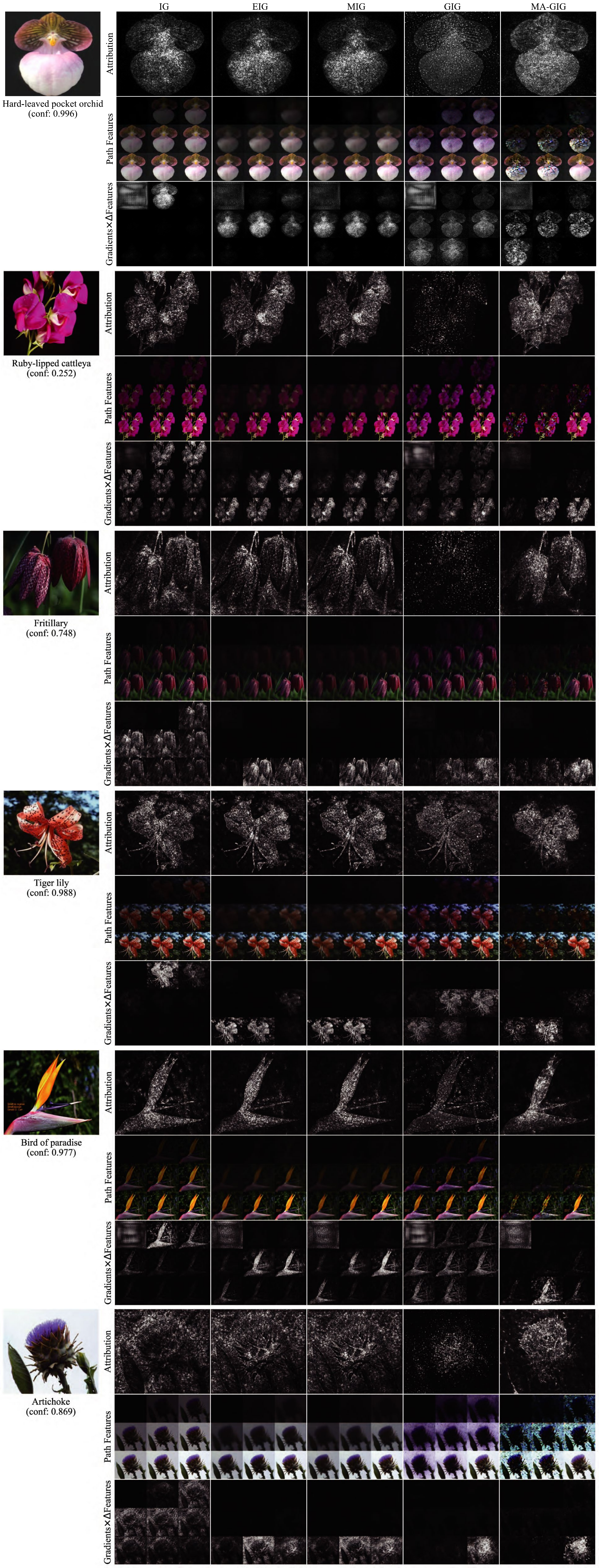}
    \hfill
    \includegraphics[width=0.48\linewidth]{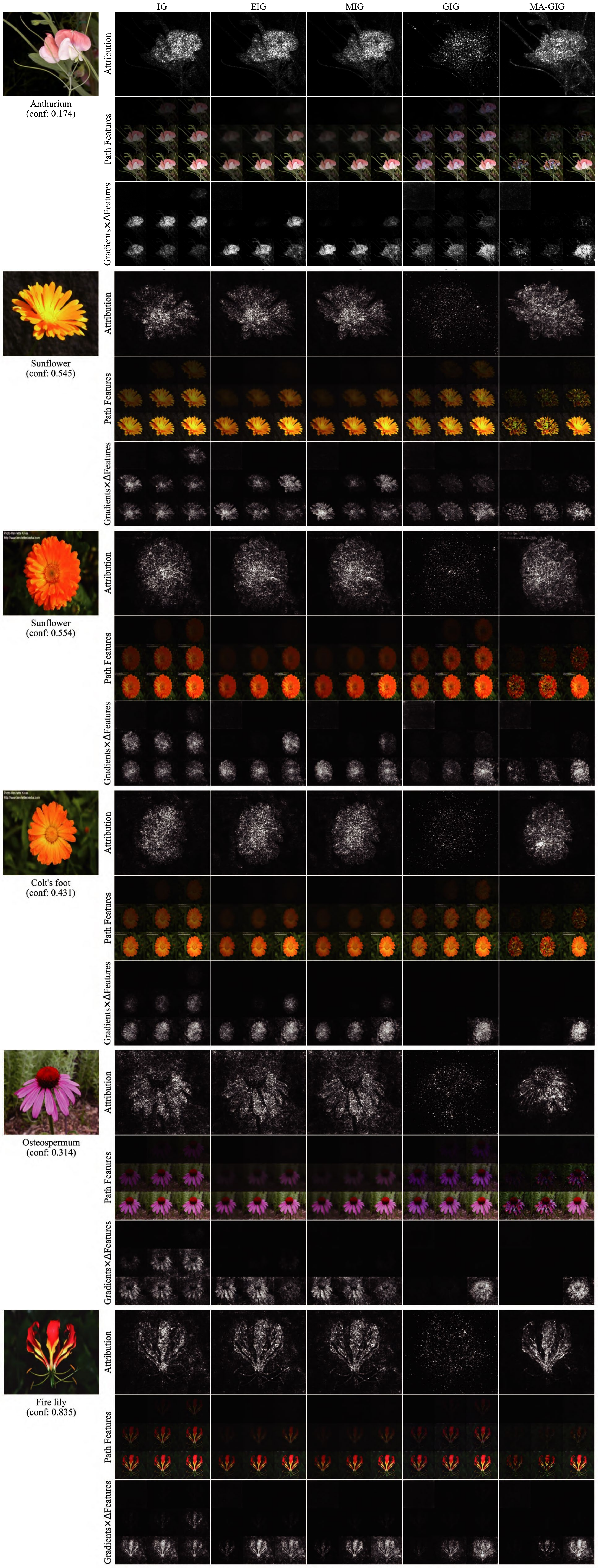}
    \caption{\textbf{Qualitative comparison on Oxford 102 flower (ResNet18 \& InceptionV1).} For each example, the top row presents the attribution maps of IG, EIG, MIG, GIG, and MA-GIG. The second and third rows visualize the evolution of path features and their corresponding gradients, sampled at nine equally spaced intervals along the integration path, demonstrating how MA-GIG aggregates relevant attributions. (Conf.: Confidence)}
    \label{fig:add_qual_4}
\end{figure}


\clearpage

\section{Manifold Adherence Analysis}
\label{app:manifold}

\subsection{LPIPS-Based Path Plausibility Analysis}
\label{app:manifold_ext}

We extend the LPIPS-based path analysis presented in Section~\ref{sec:manifold_analysis} to Oxford-IIIT Pet and Oxford 102 Flower. As illustrated in Figure~\ref{fig:manifold_appendix}, {MA-GIG (green)} maintains consistently lower perceptual deviation throughout the integration path compared to {GIG (red)}, mirroring the results observed on ImageNet.

Table~\ref{tab:manifold_summary} summarizes the quantitative results across all three datasets using ResNet18. MA-GIG consistently achieves an approximate 28\% reduction in the Area Under the Curve (AUC) for LPIPS distance. These results support that MA-GIG yields more perceptually coherent paths, although LPIPS should be interpreted as a proxy rather than direct evidence of true manifold membership.

\begin{figure}[ht]
    \centering
    \begin{subfigure}[b]{0.48\linewidth}
        \centering
        \includegraphics[width=\linewidth]{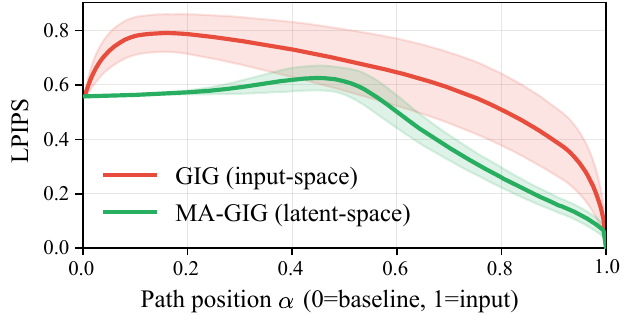}
        \caption{Oxford-IIIT Pet}
        \label{fig:manifold_oxfordpet}
    \end{subfigure}
    \hfill
    \begin{subfigure}[b]{0.48\linewidth}
        \centering
        \includegraphics[width=\linewidth]{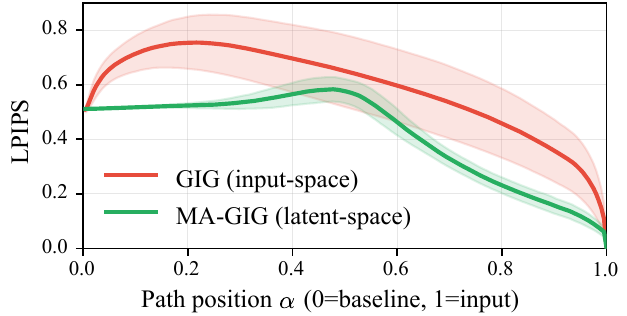}
        \caption{Oxford 102 Flower}
        \label{fig:manifold_oxfordflower}
    \end{subfigure}
    \caption{\textbf{LPIPS distance along the integration path on fine-grained datasets.} We measure LPIPS-based perceptual deviation of each intermediate sample $\gamma(\alpha)$ for (a) Oxford-IIIT Pet and (b) Oxford 102 Flower using ResNet18.}
    \label{fig:manifold_appendix}
\end{figure}

\begin{table}[ht]
    \centering
    \caption{\textbf{LPIPS-based path plausibility summary.} We report the mean Area Under the Curve (AUC) of the LPIPS distance trajectory. Lower AUC indicates lower perceptual deviation along the path.}
    \label{tab:manifold_summary}
    \begin{tabular}{lc>{\columncolor{lightgray!90}}cc}
        \toprule
        \textbf{Dataset} & \textbf{GIG AUC} & \textbf{MA-GIG AUC} & \textbf{Improvement} \\
        \midrule
        ImageNet & $0.665 \pm 0.105$ & $\mathbf{0.475} \pm 0.029$ & + 28.5\% \\
        Oxford-IIIT Pet & $0.631 \pm 0.095$ & $\mathbf{0.458} \pm 0.024$ & + 27.4\% \\
        Oxford 102 Flower & $0.588 \pm 0.102$ & $\mathbf{0.420} \pm 0.019$ & + 28.7\% \\
        \bottomrule
    \end{tabular}
\end{table}

\subsection{Distributional Diagnostics Beyond LPIPS}
\label{app:manifold_distributional}

Because LPIPS is a perceptual proxy, we further compare intermediate-path samples against the real-image distribution at $K=200$ steps on ImageNet/ResNet18. These diagnostics do not prove exact manifold membership, but provide relative evidence that MA-GIG path samples remain closer to real-image support than GIG under the same step budget.

For each method, we collect the intermediate samples along the attribution path and embed them in feature spaces defined by the ResNet18 classifier and CLIP. In the classifier feature space, we compute PRDC recall and coverage using kNN-based support estimates~\citep{kynkaanniemi2019improved,naeem2020reliable}, measure the fraction of path points falling inside the real-image kNN support as support ratio, and compute Mahalanobis AUC from a Gaussian fit to real-image features~\citep{mahalanobis1936generalized,lee2018simple}. In the CLIP feature space~\citep{radford2021learning}, we repeat the kNN support test to assess whether the same trend holds under an external semantic embedding. AUC values are computed over the normalized path index.

\begin{table}[ht]
    \centering
    \caption{\textbf{Distributional diagnostics for intermediate path samples on ImageNet/ResNet18.} Relative improvement (Rel. Improvement) is computed directionally, so higher values indicate better performance for both $\uparrow$ and $\downarrow$ metrics.}
    \label{tab:distributional_diagnostics}
    \begin{tabular}{lc>{\columncolor{lightgray!90}}cc}
        \toprule
        \textbf{Metric} & \textbf{GIG} & \textbf{MA-GIG} & \textbf{Rel. Improvement} \\
        \midrule
        PRDC Coverage ($\uparrow$) & 0.7047 & \textbf{0.7445} & + 5.65\% \\
        PRDC Recall ($\uparrow$) & 0.1378 & \textbf{0.2055} & + 49.13\% \\
        Support Ratio ($\uparrow$) & 0.4024 & \textbf{0.5658} & + 40.61\% \\
        Mahalanobis AUC ($\downarrow$) & 29.9345 & \textbf{29.8362} & + 0.33\% \\
        CLIP kNN AUC, $k=10$ ($\downarrow$) & 0.7641 & \textbf{0.7155} & + 6.36\% \\
        \bottomrule
    \end{tabular}
\end{table}

PRDC recall and coverage measure support overlap in classifier feature space. Support ratio is the fraction of path points inside a real-image kNN support threshold. Mahalanobis AUC measures distance to a Gaussian fit of real-image features, and CLIP kNN repeats the support test in an external semantic embedding space. CLIP kNN with $k=1$ and $k=5$ shows the same trend ($0.7063 \to 0.6589$ and $0.7445 \to 0.6963$). Energy scores were mixed, so we treat them only as a complementary diagnostic rather than primary evidence of manifold adherence.

\subsection{Classifier-Confidence-Based Manifold Alignment Analysis}
\label{app:manifold_ext2}

While LPIPS measures perceptual similarity, it acts primarily as a proxy for visual smoothness and does not explicitly guarantee that intermediate states remain within the data distribution recognized by the classifier. 
To further assess whether the generated path better aligns with data-supported regions
, we conducted an additional experiment analyzing the classifier confidence along the integration path.

We measured the average softmax probability (confidence) of the target class across the interpolation steps $\alpha \in [0, 1]$, averaged over 100 samples. As illustrated in \Cref{fig:manifold_cls_appendix}, the baseline {GIG (red)} exhibits prolonged uncertainty in the intermediate regions, indicating that the path largely traverses off-manifold areas and only recovers classifier trust near the input. In contrast, {MA-GIG (green)} mitigates this issue by demonstrating an earlier and more robust recovery of target class probability. This suggests that performing guidance in the latent space helps the integration path adhere better to the high-density regions of the data distribution compared to input-space guidance.

Table~\ref{tab:manifold_conf_summary} summarizes the quantitative analysis of manifold alignment based on classifier confidence. We report the Area Under the Curve (AUC) of the confidence score trajectory across the different classifiers. As shown in the results, MA-GIG consistently achieves higher AUC values compared to GIG across all evaluated models. Specifically, using ResNet18, MA-GIG shows an improvement of 2.43\% in confidence AUC, demonstrating that our method higher classifier confidence throughout the integration path.

\begin{figure}[ht]
    \centering
    \begin{subfigure}[b]{0.34\linewidth}
        \centering
        \includegraphics[width=\linewidth]{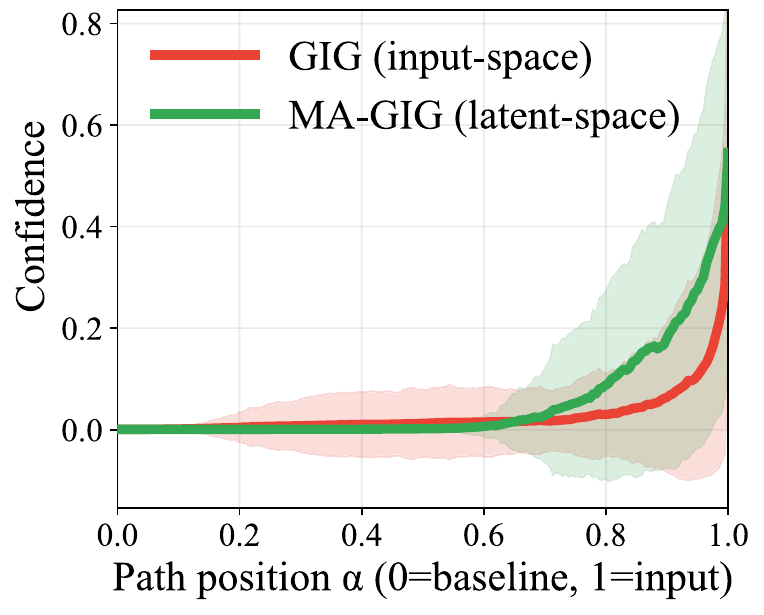}
        \caption{ResNet18}
        \label{fig:manifold_cls_resnet}
    \end{subfigure}
    \hfill
    \begin{subfigure}[b]{0.32\linewidth}
        \centering
        \includegraphics[width=\linewidth]{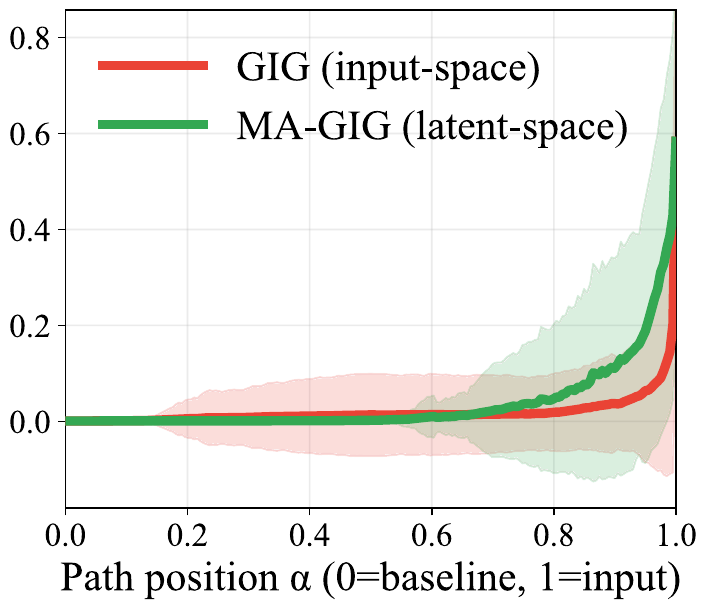}
        \caption{VGG16}
        \label{fig:manifold_cls_vgg}
    \end{subfigure}
    \hfill
    \begin{subfigure}[b]{0.32\linewidth}
        \centering
        \includegraphics[width=\linewidth]{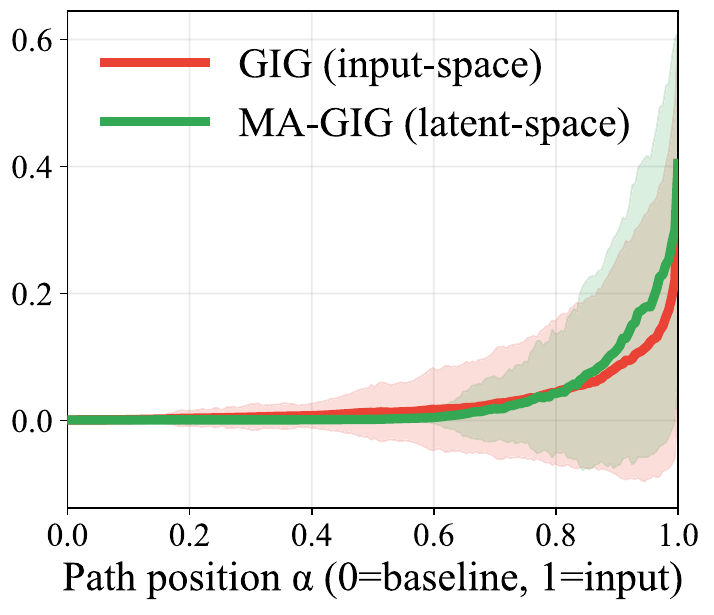}
        \caption{InceptionV1}
        \label{fig:manifold_cls_inception}
    \end{subfigure}
    \caption{\textbf{Average classifier confidence along the integration path ($\alpha$ from 0 to 1) on three classifiers.} We measure the softmax score of each intermediate sample $\gamma(\alpha)$ for (a) ResNet18, (b) VGG16, and (c) InceptionV1 on ImageNet2012.
    }
    \label{fig:manifold_cls_appendix}
\end{figure}

\begin{table}[ht]
    \centering
    \caption{\textbf{Classifier-Confidence-Based manifold alignment summary.} We report the mean Area Under the Curve (AUC) of the classifier confidence score trajectory. Higher AUC indicates better manifold alignment.}
    \label{tab:manifold_conf_summary}
    \begin{tabular}{lc>{\columncolor{lightgray!90}}cc}
        \toprule
        \textbf{Classifier} & \textbf{GIG AUC} & \textbf{MA-GIG AUC} & \textbf{Improvement} \\
        \midrule
        ResNet18    & $0.0267 \pm 0.0581$ & $\mathbf{0.0511} \pm 0.0634$ & + 2.43\% \\
        VGG16       & $0.0193 \pm 0.0591$ & $\mathbf{0.0369} \pm 0.0542$ & + 1.75\% \\
        InceptionV1 & $0.0288 \pm 0.0543$ & $\mathbf{0.0314} \pm 0.0386$ & + 0.26\% \\
        \bottomrule
    \end{tabular}
\end{table}

\section{Computational Cost \& Runtime}
\label{app:cost}

We report wall-clock runtime measurements to contextualize the computational overhead of MA-GIG relative to existing methods. All experiments were conducted on a system equipped with four NVIDIA B200 192GB GPUs.

Table~\ref{tab:runtime} presents the runtime comparison alongside attribution faithfulness metrics on Oxford-IIIT Pet with ResNet18. While MA-GIG incurs higher computational cost than simple gradient-based methods such as IG, it remains competitive with optimization-based approaches like IG$^2$ and GIG. Notably, MA-GIG achieves substantially faster inference than MIG, which requires iterative geodesic computation in latent space. We argue that the moderate runtime overhead of MA-GIG is justified by its significant improvements in faithfulness metrics: MA-GIG achieves the highest DiffID score (0.4637) while requiring only ${\sim}5$ seconds per image, compared to MIG which takes ${\sim}47$ seconds yet yields lower faithfulness.

\begin{table}[ht]
    \centering
    \caption{\textbf{Runtime analysis on Oxford-IIIT Pet (ResNet18).} We report attribution faithfulness metrics alongside wall-clock time averaged over 10 samples (mean $\pm$ std). Best results are \textbf{bolded}.}
    \label{tab:runtime}
    \vspace{2mm}
    \footnotesize
    \setlength{\tabcolsep}{3pt}
    \begin{tabular}{l cccc}
        \toprule
        \textbf{Method} & \textbf{DiffID ($\uparrow$)} & \textbf{Ins ($\uparrow$)} & \textbf{Del ($\downarrow$)} & \textbf{Time (sec/img)} \\
        \midrule
        AGI~\citep{pan2021explaining} & 0.2787 & 0.4453 & 0.1667 & \textbf{0.16} $\pm$ 0.03 \\
        IG~\citep{sundararajan2017axiomatic} & 0.3790 & 0.5186 & 0.1396 & 0.97 $\pm$ 0.12 \\
        IG$^2$~\citep{zhuo2024ig2} & 0.3823 & 0.5264 & 0.1441 & 1.68 $\pm$ 0.18 \\
        GIG~\citep{kapishnikov2021guided} & 0.3634 & 0.5093 & 0.1459 & 2.17 $\pm$ 0.26 \\
        EIG~\citep{jha2020enhanced} & 0.3595 & 0.4964 & 0.1369 & 2.27 $\pm$ 0.10 \\
        MIG~\citep{zaher2024manifold} & 0.3486 & 0.4889 & 0.1402 & 46.53 $\pm$ 1.35 \\
        \midrule
        \rowcolor{gray!15}
        MA-GIG & \textbf{0.4637} & \textbf{0.5886} & \textbf{0.1249} & 5.10 $\pm$ 0.18 \\
        \bottomrule
    \end{tabular}
\end{table}



\end{document}